\documentclass[letterpaper,11pt]{article}
\usepackage[margin=1in]{geometry}
\usepackage[bookmarks, colorlinks=true, plainpages = false, citecolor = blue,linkcolor=red,urlcolor = blue, filecolor = blue]{hyperref}
\usepackage{url}
\usepackage{amsmath,amsfonts,amsthm,amssymb,bm, verbatim,dsfont,mathtools}
\usepackage{color,graphicx,appendix}
\usepackage{subfigure}
\usepackage{etoolbox}
\usepackage{tikz}
\usepackage{xr,xspace}
\usepackage{todonotes}
\usepackage{paralist}
\usepackage{caption,soul}
\usepackage{algorithm}
\usepackage{algorithmic}
\makeatletter

\newtheorem{theorem}{Theorem}[section]
\newtheorem{lemma}[theorem]{Lemma}
\newtheorem{proposition}[theorem]{Proposition}
\newtheorem{corollary}[theorem]{Corollary}
\newtheorem{definition}[theorem]{Definition}

\newtheorem{conjecture}[theorem]{Conjecture}

\usepackage{xspace,prettyref}

\newcommand{\reals}{{\mathbb{R}}}

\newcommand{\naturals}{{\mathbb{N}}}

\newcommand{\polylog}{{\rm{polylog}}}

\newcommand{\eexp}{{\rm e}}

\newcommand{\diff}{{\rm d}}

\newcommand{\expect}[1]{\mathbb{E}\left[ #1 \right]}
\newcommand{\expectp}[1]{\mathbb{E}^{+} \left[ #1 \right]}
\newcommand{\expectm}[1]{\mathbb{E}^{-}\left[ #1 \right]}

\newcommand{\prob}[1]{ \mathbb{P}\left\{ #1 \right\} }

\newcommand{\var}{\mathsf{var}}

\newcommand{\Binom}{{\rm Binom}}
\newcommand{\Pois}{{\rm Pois}}

\newcommand{\eg}{e.g.\xspace}
\newcommand{\ie}{i.e.\xspace}

\newrefformat{eq}{(\ref{#1})}
\newrefformat{chap}{Chapter~\ref{#1}}
\newrefformat{sec}{Section~\ref{#1}}
\newrefformat{alg}{Algorithm~\ref{#1}}
\newrefformat{fig}{Fig.~\ref{#1}}
\newrefformat{tab}{Table~\ref{#1}}
\newrefformat{rmk}{Remark~\ref{#1}}
\newrefformat{clm}{Claim~\ref{#1}}
\newrefformat{def}{Definition~\ref{#1}}
\newrefformat{cor}{Corollary~\ref{#1}}
\newrefformat{lmm}{Lemma~\ref{#1}}
\newrefformat{prop}{Proposition~\ref{#1}}
\newrefformat{app}{Appendix~\ref{#1}}
\newrefformat{hyp}{Hypothesis~\ref{#1}}
\newrefformat{thm}{Theorem~\ref{#1}}

\newcommand{\indc}[1]{{\mathbf{1}_{\left\{{#1}\right\}}}}

\newcommand{\calE}{{\mathcal{E}}}
\newcommand{\calF}{{\mathcal{F}}}
\newcommand{\calG}{{\mathcal{G}}}

\newcommand{\calN}{{\mathcal{N}}}

\newcommand{\ER}{Erd\H{o}s-R\'enyi\xspace}

\renewcommand{\tilde}{\widetilde}
\renewcommand{\hat}{\widehat}

\newcommand{\sgn}{\mathrm{sgn}}

\begin{document}

\title{Local Algorithms for Block Models with Side Information}
\date{\today}
\author{
Elchanan Mossel\thanks{Research supported by NSF grants CCF 1320105, DOD ONR grant N00014-14-1-0823, and grant 328025 from the Simons Foundation. E.M is with Department of Statistics, The Wharton School, University of Pennsylvania, Philadelphia, PA and with 
the Departments of Statistics and Computer Science, U.C. Berkeley, Berkeley CA, \texttt{mossel@wharton.upenn.edu}.
}
 \and Jiaming Xu\thanks{Research supported by
DOD ONR Grant N00014-14-1-0823, and Grant 328025 from the Simons Foundation.
J. X is with Department of Statistics, The Wharton School, University of Pennsylvania, Philadelphia, PA, \texttt{jiamingx@wharton.upenn.edu}.
}}

\maketitle

\begin{abstract}
There has been a recent interest in understanding the power of local algorithms
for optimization and inference problems on sparse graphs. Gamarnik and Sudan (2014) showed that local algorithms are weaker than global algorithms for finding large independent sets in sparse random regular graphs thus refuting a conjecture by Hatami, Lov\'asz, and Szegedy (2012).
Montanari (2015) showed that local algorithms are  suboptimal for finding a community with high connectivity in the sparse \ER random graphs. 
For the symmetric planted partition problem (also named community detection for the block models) on sparse graphs, a simple observation is that  local algorithms cannot have 
non-trivial performance.

In this work we consider the effect of \emph{side information} on local algorithms for community detection under the binary symmetric stochastic block model.
In the block model with side information 
each of the $n$ vertices is labeled $+$ or $-$ independently
and uniformly at random; each pair of vertices is connected independently with probability $a/n$ if both of them
have the same label or $b/n$ otherwise. The goal is to estimate the
underlying vertex labeling given 1) the graph structure and 2) side information in the form of a vertex labeling positively correlated
with the true one. Assuming that the ratio between in and out degree $a/b$ is
$\Theta(1)$ and the average degree $ (a+b) / 2 = n^{o(1)}$, 
we show that a local algorithm, namely, belief propagation  run on the local neighborhoods, maximizes the expected fraction of vertices
labeled correctly in the following three regimes:
\begin{itemize}
\item $|a-b|<2$ and all $0 < \alpha < 1/2$ 
\item $(a-b)^2 > C (a+b)$ for some constant $C$ and all $0 < \alpha < 1/2$ 
\item For all $a,b$ if the probability that each given vertex label is incorrect is at most $\alpha^\ast$ for some constant $\alpha^\ast \in (0,1/2)$. 
\end{itemize}
Thus, in contrast to the case of independent sets or a single community in random graphs and to the case of symmetric block models without side information,
we show that local algorithms achieve optimal performance in the above three regimes for the block model with side information.

To complement our results, in the large degree limit $a \to \infty$, we give a formula of the expected fraction of vertices
labeled correctly by the local belief propagation, in terms of a fixed point of a recursion derived from the density evolution analysis with Gaussian approximations.
\end{abstract}

\newpage

\section{Introduction}

In this work we study the performance of {\em local algorithms} for {\em community detection}
in sparse graphs thus combining two lines of work which saw recent breakthroughs.

The optimality of the performance of local algorithm for  optimization problems on large graphs was raised by  Hatami, Lov\'asz, and Szegedy~\cite{HaLoSz:12} in the context of a theory of graph limits for sparse graphs.   The conjecture, regarding the optimality for finding independent sets in random graphs was refuted by Gamarnik and Sudan~\cite{GamarnikSudan:14}. More recently, Montanari~\cite{Montanari:15OneComm} showed that local algorithms are strictly suboptimal comparing to the global exhaustive search for finding a community with high connectivity in the sparse \ER random graph. 

In a different direction, following a beautiful conjecture from physics~\cite{DKMZ:2011}, new efficient  algorithms for the stochastic block models (i.e.\ planted partition) were developed and shown to detect the blocks whenever this is information theoretically possible~\cite{MONESL:15,MNS:2013b,Massoulie:2013,BordenaveLelargeMassoulie:2015dq}. 
It is easy to (see e.g.~\cite{KaMoSc:14}) that no local algorithm with access to neighborhoods of radius $o(\log n)$
 can have non-trivial performance for this problem. 


Our interest in this paper is in the application of local algorithms for community detection with side information on community structures. The motivations are two-folded: 1) from a theoretical perspective it is interesting to ask what is the effect of side information on the existence of optimal local algorithms 2) from the application perspective, it is important to know how to efficiently exploit side information in addition to the graph structure for community detection. 
We show that unlike the cases of independent sets on regular graphs or the case of community detection on sparse random graphs, local algorithms do have optimal performance.

\subsection{Local algorithms}


Local algorithms for optimization problems on sparse graphs are algorithms that determine if each vertex belongs to the solution or not based only on a small radius neighborhood around the node. Such algorithms are allowed to have an access to independent random variables associated to each node.

A simple example for a local algorithm is the following classical algorithm for finding independent sets in graphs. Attach to each node $v$ an independent uniform random variable $U_v$. Let the independent set consist of all the vertices whose $U_v$ value is greater than that of all of their neighbors. See Definition \ref{def:localalgorithm} for a formal definition of a local algorithm and ~\cite{LyonsNazarov:11,HaLoSz:12,GamarnikSudan:14, Montanari:15OneComm} for more background on  local algorithms.

There are many motivations for studying local algorithms: These algorithms are efficient: for example, for bounded degree graphs they run in linear time in the size of the graph and for graphs with maximal degree $\polylog(n)$ they run in time $n \times \polylog(n)$. Moreover, by design, these algorithms are easy to run in a distributed fashion.
Moreover, the existence of local algorithms implies correlation decay properties that are of interest in statistical physics, graph limit thoery and ergodic theory. Indeed the existence of a local algorithm implies that the solution in one part of the graph is independent of the solution in a far away part, see~\cite{LyonsNazarov:11,HaLoSz:12,GamarnikSudan:14} for a more formal and comprehensive discussion.

A striking conjecture of   Hatami, Lov\'asz, and Szegedy~\cite{HaLoSz:12} stated that local algorithms are able to find independent sets of the maximal possible density in random regular graphs.  
This conjecture was refuted by Gamarnik and Sudan~\cite{GamarnikSudan:14}. 
The work of Gamarnik and Sudan~\cite{GamarnikSudan:14} highlights the role of long range correlation and clustering in the solution space as obstacles for the optimality of local algorithms. Refining the methods of Gamarnik and Sudan, Rahman and Virag~\cite{RahmanVirag:14} showed that 
local algorithms cannot find independent sets of size larger  than half of the optimal density.

\subsection{Community detection in sparse graphs}

The stochastic block model is one of the most popular models for networks with
clusters.  The model has been extensively studied in statistics~\cite{HLL:1983,
SN:1997, BC:2009,cai2014robust,Zhangzhou15,Gao15}, computer science (where it is called the planted partition
problem)~\cite{DF:1989,JS:1998,CK:2001,McSherry:2001,Coja-Oghlan05,Coja-oghlan10,Chen12, anandkumar2013tensormixed,ChenXu14} and theoretical
statistical physics~\cite{DKMZ:2011, ZKRZ:2012, DKMZ:2011a}. In the simplest binary symmetric form,
it assumes that $n$ vertices are assigned into two clusters, or equivalently labeled with $+$ or $-$,
 independently and uniformly at random;
each pair of vertices is connected independently with probability $a/n$ if both of them
are in the same cluster or $b/n$ otherwise.

In the dense regime with $a=\Omega(\log n)$, it is possible to exactly recover the clusters from the observation of the graph. 
 A sharp exact recovery threshold has been found in \cite{Abbe14,Mossel14} and it is further shown that semi-definite programming can achieve the sharp threshold in \cite{HajekWuXuSDP14,Bandeira15}.
More recently, exact recovery thresholds have been identified in a more general setting with a fixed number of clusters \cite{HajekWuXuSDP15,YunProutiere14}, and with heterogeneous cluster sizes and edge probabilities \cite{AbbeSandon15,PerryWein15}.

Real networks are often sparse with bounded average degrees.
In the sparse setting with $a=\Theta(1)$, exact recovery of the clusters from the graph becomes hopeless as
the resulting graph under the stochastic block model will have many isolated vertices. Moreover, it is easy to see that even
vertices with constant degree cannot be labeled accurately given all the other
vertices' labels are revealed. Thus the goal in the sparse regime is to find a labeling that has a non-trivial
or maximal correlation with the true one (up to permutation of cluster labels).
It was conjectured in~\cite{DKMZ:2011} and
proven in~\cite{MONESL:15,MNS:2013b,Massoulie:2013} that nontrivial detection is feasible if and only if $(a-b)^2>2(a+b)$.
A spectral method based on
the non-backtracking matrix is shown to achieve the sharp threshold in \cite{BordenaveLelargeMassoulie:2015dq}.
In contrast, a simple argument in ~\cite{KaMoSc:14} shows that no local algorithm running on neighborhoods of radius $o(\log n)$
can attain nontrivial detection. 


\subsection{Community detection with side information}
The community detection problem under stochastic block model is an idealization of a network inference problem.
In many realistic settings, in addition to network information, some partial information about vertices' labels is
also available.  There has
been much recent work in the machine learning and applied networks communities
on combining vertex and network information (see for
example~\cite{ChWeSc:02,BaBaMo:02,BaBiMo:04,NewmanClauset15}).
In this paper, we ask the following natural but fundamental question:
\\

\emph{With the help of partial information about vertices' labels,
can local algorithms achieve the optimal detection probability?}
\\

This question has two motivations: 1) from a theoretical perspective we would like to understand how side information affects the existence of optimal local algorithms; 
2) from the application perspective, it is important to develop fast community detection algorithms which exploit side information in addition to the graph structure. 

There are two natural models for side information of community structures:
\begin{itemize}
\item
A model where a small random fraction of the vertices' labels is given accurately.
This model was considered in a number of recent works in physics and computer science~\cite{DKMZ:2011,vSMGE:2013,AvSG:2010,KaMoSc:14}.
The emerging conjectured picture is that in the case of  the binary symmetric stochastic block model, the local application of BP is able to achieve  the optimal detection probability.
This is stated formally as one of the main conjectures of~\cite{KaMoSc:14}, where it is proven in an asymptotic regime where the fraction of revealed information goes to $0$ {\em and}
 assuming $(a-b)^2 > C(a+b)$ for some large constant $C$.
\item
The model considered in this paper is where noisy information is provided for each vertex.
Specifically, for each vertex,  we observe a noisy label which is the same as its true label with probability $1-\alpha$
and different with probability $\alpha$, independently at random, for some $\alpha \in [0,1/2)$.

For this model, by assuming that $a/b=\Theta(1)$ and the average degree $ (a+b) / 2 = n^{o(1)}$ is smaller than all powers of $n$, we prove that local application of belief propagation maximizes the expected fraction of vertices labelled correctly, \ie, achieving the optimal detection probability, in the following regimes
\begin{itemize}
\item $|a-b|<2$,
\item $(a-b)^2 > C(a+b)$ for some constant $C$,
\item  $\alpha \le \alpha^\ast$ for some constant $0<\alpha^\ast < 1/2$.
\end{itemize}
Note that this proves the conjectured picture in a wide range of the parameters.
In particular, compared to the results of~\cite{KaMoSc:14}, we prove the conjecture in the whole regime $((a-b)^2 > C(a+b)) \times (\alpha \in (0,1/2))$ while in~\cite{KaMoSc:14}
the result is only proven for the limiting interval of this region $((a-b)^2 > C(a+b)) \times (\alpha' \to 0^{+})$,
where each vertex's true label is revealed with probability $\alpha'$.

In the large degree limit $a \to \infty$
we further provide a simple formula of the expected fraction of vertices labeled correctly by BP, in terms of a fixed point
of a recursion, based on the density evolution analysis. Density evolution has been used for the analysis of sparse
graph codes \cite{Urbanke08,Mezard09}, and more recently for the analysis of finding a single community in a sparse graph \cite{Montanari:15OneComm}.
\end{itemize}

\section{Model and main results}
We next present a formal definition of the model followed by a formal statement of the main results.

\subsection{Model}
We consider the binary symmetric stochastic block model with two clusters.
This is a random graph model on $n$ vertices,
where we first independently assign each vertex into one of the clusters uniformly at random, and then
independently draw an edge between each pair of vertices with probability $a/n$ if two
vertices are in the same clusters or $b/n$ otherwise. Let $\sigma_i=+$ if vertex $i$ is in the first cluster and $\sigma_i=-$ otherwise.

Let $G=G_n=(V,E)$ denote the observed graph (without the labels $\sigma$).
Let $\tilde{\sigma}$ be an $\alpha$ noisy version of $\sigma$:
for each vertex $i$ independently, $\tilde{\sigma}_i=\sigma_i$ with probability $1-\alpha$
and $\tilde{\sigma}_i=-\sigma_i$ with probability $\alpha$, where $\alpha \in [0, 1/2)$ is a fixed constant.
Hence, $\tilde{\sigma}$ can be viewed as the side information for the cluster structure.

\begin{definition}
The {\em detection problem with side information} is the inference problem of inferring $\sigma$
 from the observation of $(G, \tilde{\sigma})$. The estimation accuracy for an estimator $\hat{\sigma}$ is defined by
\begin{align}
p_{G_n}(\hat{\sigma})= \frac{1}{n}  \sum_{i=1}^n \prob{\sigma_i=\hat{\sigma}_i}, \label{eq:estimationaccuracy}
\end{align}
which equals to the expected fraction of vertices labeled correctly. Let $p^\ast_{G_n}$  denote the optimal estimation accuracy.
\end{definition}
The optimal estimator in maximizing the success probability $ \prob{\sigma_i=\hat{\sigma}_i}$ is the maximum a posterior (MAP) estimator, which is given by
$
 2 \times \indc{ \prob{\sigma_i=+  | G, \tilde{\sigma} } \ge  \prob{\sigma_i=-  |G, \tilde{\sigma}  } } -1,
$
and the maximum success probability is $ \frac{1}{2}  \expect{ | \prob{\sigma_i=+  | G, \tilde{\sigma} } -  \prob{\sigma_i=-  |G, \tilde{\sigma} } |} +\frac{1}{2}$.
Hence, the optimal estimation accuracy $p^\ast_{G_n}$ is given by
\begin{align}
p_{G_n} ^\ast  & = \frac{1}{2n}  \sum_{i=1}^n \expect{ \big| \prob{\sigma_i=+  | G, \tilde{\sigma} } -  \prob{\sigma_i=-  |G, \tilde{\sigma} } \big|} + \frac{1}{2} \nonumber  \\
&= \frac{1}{2}  \expect{ | \prob{\sigma_i=+  | G, \tilde{\sigma} } -  \prob{\sigma_i=-  |G, \tilde{\sigma} } |} +\frac{1}{2}, \label{eq:optimalaccuracygraph}
\end{align}
where the second equality holds due to the symmetry. However, computing the MAP estimator is computationally intractable in general, and it is unclear
whether the optimal estimation accuracy $p^\ast_{G_n}$ can be achieved by some estimator computable in polynomial-time.


In this paper, we focus on the regime:
\begin{align}
\frac{a}{b} = \Theta(1), \quad   a =n^{o(1)}, \quad  \text{ as } n \to \infty,  \label{eq:asymptotics}
\end{align}

It is well know that in the regime $a=n^{o(1)}$, a local neighborhood of a vertex is with high probability a tree. Thus, it is natural
to study the performance of local algorithms. We next present a formal definition of local algorithms which is a slight variant of the definition in~\cite{Montanari:15OneComm}.

Let $\calG_\ast$ denote the space of graphs with one distinguished vertex and labels $+$ or $-$ on each vertex.  
For an estimator $\hat{\sigma}$,
it can be viewed as a function $\hat{\sigma} : \calG_\ast \to \{ \pm \}$, which maps $(G, \tilde{\sigma}, u)$ to $\hat{\sigma}_u$
for every $(G, \tilde{\sigma}, u ) \in \calG_\ast.$
\begin{definition}\label{def:localalgorithm}
Given a $t \in \naturals$, an estimator $\hat{\sigma}$ is $t$-local if there exist a function $\calF: \calG_\ast \to \{\pm\}$ such that 
for all $(G, \tilde{\sigma}, u ) \in \calG_\ast$, 
\begin{align*}
\hat{\sigma} ( G, \tilde{\sigma}, u ) = \calF( G_u^t, \tilde{\sigma}_{G_u^t} ), 
\end{align*}
where $G_u^t$ is the  subgraph of $G$ induced by vertices whose distance to  $u$ is at most $t$; the distinguished vertex is $u$,
and  each vertex $i$ in $G_u^t$ has label $\tilde{\sigma}_i$; $\tilde{\sigma}_{G_u^t}$ is the restriction of $\tilde{\sigma}$ to vertices 
in $G_u^t$. Moreover, we call an estimator $\hat \sigma$ local, if it is $t$-local for some fixed $t$, regardless of the graph size $n$. 
\end{definition}
We can potentially allow local algorithms to access local independent uniform random variables as defined in \cite{HaLoSz:12,GamarnikSudan:14}. Since our main results show that the local BP algorithm which does not need to access external randomness, is already optimal, the extra randomness is not needed in our context.



\subsection{Local belief propagation algorithm} \label{sec:BP}
It is well known that, see \eg \cite[Lemma 4.3]{Montanari:15OneComm}, local belief propagation algorithm as defined in \prettyref{alg:MP_commun}
maximizes the estimation accuracy among local algorithms, provide that the graph is locally tree-like. Thus we focus on studying the local BP. 
\begin{algorithm}[htb]
\caption{Local belief propagation with side information}\label{alg:MP_commun}
\begin{algorithmic}[1]
\STATE Input: $n \in \naturals,$ $a>b >0$, $\alpha \in [0,1/2)$, adjacency matrix $A \in \{0,1\}^{n\times n}$,  and $t \in \naturals$.
\STATE Initialize: Set  $R^{0}_{i \to j}=0$ for all $i \in [n]$ and $ j \in \partial i$.
\STATE Run $t-1$ iterations of message passing as in \prettyref{eq:mp_commun} to compute $R^{t -1 }_{i\to j}$ for all $i \in [n]$ and $j \in \partial i$.
\STATE Compute  $R_{i}^{t}$ for all $i \in [n]$ as per \prettyref{eq:mp_combine_commun}.
\STATE Return $\hat{\sigma}_{\rm BP}^t$ with $\hat{\sigma}^t_{\rm BP} (i)= 2 \times \indc{ R^{t}_i \ge 0 }  -1.$
\end{algorithmic}
\end{algorithm}
Specifically, let $\partial i$ denote the set of neighbors of $i$ and define
\begin{align}
R_{i \to j}^t  = h_i + \sum_{\ell \in \partial i \backslash \{j\} }  F (R^{t-1}_{\ell \to i} ), \label{eq:mp_commun}
\end{align}
with initial conditions $R_{i \to j}^{0} = \gamma $ if $\tau_i=+$ and $R_{i \to j}^{0} =- \gamma$ if $\tau_i=-$, for all $i \in [n]$ and $j \in \partial i$.
Then we approximate $\frac{1}{2} \log \frac{\prob{G, \tilde{\sigma} | \sigma_u=+ } }{\prob{G, \tilde{\sigma} | \sigma_u=-}}$
by $R_u^t$ given by
\begin{align}
R_{u}^t  = h_u + \sum_{\ell \in \partial u }  F(\Lambda^{t-1}_{\ell \to u} ) \label{eq:mp_combine_commun}.
\end{align}
We remark that in each BP iteration, the number of outgoing messages to compute is $O(|E|)$, where $|E|$ is the total number of edges; each outgoing message needs to process $d'$ incoming
messages on average, where $d'$ is the average number of edges incident to an edge chosen uniformly at random.
Thus each BP iteration runs in time $O(|E| d')$ and $\hat{\sigma}_{\rm BP}^t$ is computable in time $O(t |E| d' )$. In the sparse graph with $a=\Theta(1)$,
$\hat{\sigma}_{\rm BP}^t$ runs in time linear in the size of the graphs. For graphs with maximal degree $\text{polylog}(n)$, it runs in time $n\; \text{polylog}(n).$

\subsection{Main results}
\begin{theorem}\label{thm:optimality}
Consider the detection problem with side information assuming that $a/b = \Theta(1)$ and
that $a = n^{o(1)}$.
Let $\hat{\sigma}_{\rm BP}^t$ denote the estimator given by Belief Propagation applied for $t$ iterations, as defined in \prettyref{alg:MP_commun}.
Then
\[
\lim_{t \to \infty} \limsup_{n \to \infty}  \left( p_{G_n}^\ast - p_{G_n}(\hat{\sigma}_{\rm BP}^t  ) \right) = 0
\]
in the following three regimes:
\begin{itemize}
\item $|a-b|<2$,
\item $(a-b)^2 > C(a+b)$ for some constant $C$,
\item $\alpha \le \alpha^\ast $ for some $0<\alpha^\ast <1/2$.
\end{itemize}
In other words, in each of these regimes a local application of belief propagation provides an optimal detection probability. 
\end{theorem}
The above results should be contrasted with the case with no side information available, where it is known, see e.g.~\cite{KaMoSc:14},  that BP applied for $t = o(\log n)$ iterations  cannot recover a partition better than random, \ie, achieving
the non-trivial detection.

In the large degree regime, we further derive an asymptotic formula for $p_{G_n}(\hat{\sigma}_{\rm BP}^t )$
in terms of a fixed point of a recursion.

\begin{theorem}\label{thm:accuracy}
Consider the regime \prettyref{eq:asymptotics}. Assume further that as $n \to \infty$,
 $a \to \infty$ and $\frac{a-b}{\sqrt{b} } \to \mu,$
 where $\mu$ is a fixed constant.
 Let $h(v) = \expect{ \tanh ( v + \sqrt{v} Z + U )   }$, where $Z \sim \calN(0,1)$; $U$ is independent of $Z$ and $U= \gamma$ with probability $1-\alpha$ and $U=-\gamma$ with probability $\alpha$,
 where $\gamma= \frac{1}{2} \log \frac{1-\alpha}{\alpha}$.
 Define $\underline{v}$ and $\overline{v}$ to be the smallest and largest fixed point of $v= \frac{\mu^2}{4} h(v)$, respectively.
 let $\hat{\sigma}_{\rm BP}$ denote the estimator given by Belief Propagation applied for $t$ iterations, as defined in \prettyref{alg:MP_commun}.
  Then,
 \begin{align*}
\lim_{t \to \infty} \lim_{n \to \infty} p_{G_n } (\hat{\sigma}_{\rm BP}^t ) & = 1-  \expect{ Q \left(\frac{ \underline{v}  + U }{ \sqrt{\underline{v}  } } \right)}, \\
 \limsup_{n \to \infty}p_{G_n}^\ast & \le 1- \expect{ Q \left(\frac{ \overline{v}  + U }{ \sqrt{\overline{v}  } } \right)} ,
 \end{align*}
 where $Q(x)=\int_{x}^{\infty} \frac{1}{\sqrt{2\pi}} \eexp^{-y^2/2} \diff y$.  Moreever, $\underline{v}=\overline{v}$
 and $\lim_{t \to \infty} \lim_{n \to \infty} p_{G_n}(\hat{\sigma}_{\rm BP}^t )= \limsup_{n \to \infty}p_{G_n}^\ast$ in the following three regimes:
\begin{itemize}
\item $|\mu|<2$,
\item $|\mu|>C$ for some constant $C$,
\item $\alpha \le \alpha^\ast $ for some $0<\alpha^\ast <1/2$.
\end{itemize}
\end{theorem}

\subsection{Proof ideas}
The proof of \prettyref{thm:optimality} follows ideas from~\cite{MONESL:15,MNS:2013a}.
\begin{itemize}
\item To bound from above the accuracy of an arbitrary estimator, we bound its accuracy for a specific random vertex $u$. Following~\cite{MONESL:15}, we consider an estimator, which in addition to the graph structure and the noisy labels, the exact labels of all vertices at distance exactly $t$ from $u$ is also given. As in~\cite{MONESL:15}, it is possible to show that the best estimator in this case is given by BP for $t$ levels using the exact labels at distance $t$.
\item
The only difference between our application of BP and the BP upper bound above is the quality of information at distance exactly $t$ from vertex $u$. Our goal is to now analyze the recursion of random variables defining BP in both cases and show they converge to the same value given exact or noisy information at level $t$.
\item 
In the two cases where 1) $(a-b)^2 > C(a+b)$  and 2) where $\alpha$ is small, our proof follows the pattern of~\cite{MNS:2013a}. We note however that the paper~\cite{MNS:2013a} did not consider side information and the adaptation of the proof is far from trivial. 
Similar to the setup in~\cite{MNS:2013a}, the noisy labels at the boundary, \ie, level $t$, play the role as an initialization of the recursion.
However, the noisy labels inside the tree results in less symmetric recursions that need to be controlled. 
Finally in the case where $\alpha$ is small they play a novel role as the reason behind the contraction of the recursion. 
\item
The case where $a-b<2$ corresponds to the uniqueness regime. Here  the recursion converges to the same value if all the vertices at level $t$ are $+$ or all vertices at level $t$ are $-$. This implies that it converges to the same value for all possible values at level $t$.
\end{itemize}
The proof of \prettyref{thm:accuracy} instead follows the idea of density evolution \cite{Urbanke08,Mezard09},
which was recently used for analyzing the problem of finding a single community in a sparse graph \cite{Montanari:15OneComm}.
\begin{itemize}
\item The neighborhood of a vertex $u$ is locally tree-like and thus the incoming messages to vertex $u$ from its neighbors  in BP iterations are independent. In the large degree limit, the sum of incoming messages is distributed as Gaussian conditional on its label. Moreover, its mean and variance admit a simple recursion over $t$, which
    converge to a fixed point as $t\to \infty$.
\item
As we pointed out earlier, the only difference between our application of BP and the BP upper bound discussed above is the quality of information at distance exactly $t$ from vertex $u$. Hence, the mean and variance for
both BPs satisfy the same recursion but with different initialization. If there is a unique fixed point of the recursion for mean and variance, then the mean and variance for
both BPs  converge to the same values as $t\to \infty$.
\item The case $|\mu|<2$ exactly corresponds to the regime below the Kesten-Stigum bound \cite{KestenStigum:66}. In this case, we can show that the recursion is a contraction mapping 
and thus has a unique fixed point. 
\end{itemize}

\subsection{Conjectures and open problems}\label{sec:numerical}
There are many interesting conjectures and open problems resulting from this work.  
First, we believe that local BP with side information always achieves optimal estimation accuracy. 
\begin{conjecture}
Under the binary symmetric stochastic block model with $\alpha$-noisy side information,
 $\lim_{t \to \infty} \lim_{n \to \infty} p_{G_n}(\hat{\sigma}_{\rm BP}^t )= \limsup_{n \to \infty}p_{G_n}^\ast$ holds for all $a$, $b$, and $\alpha$.
\end{conjecture}
In the large degree regime with $a \to \infty $, $a=n^{o(1)}$, and $ \frac{a-b}{\sqrt{b} } \to \mu$, \prettyref{thm:accuracy} implies that
the above conjecture is true if $v=\mu^2 h(v)/4$ always has a unique fixed point. Through simulations, we find that
$v=\mu^2 h(v)/4$ seems to have a unique fixed point for all $\mu$ and $\alpha$, and the asymptotically optimal estimation accuracy
is depicted in Fig.~\ref{fig:h_ErrorProb}.
\begin{figure}[ht]
\centering
\includegraphics[width=4in]{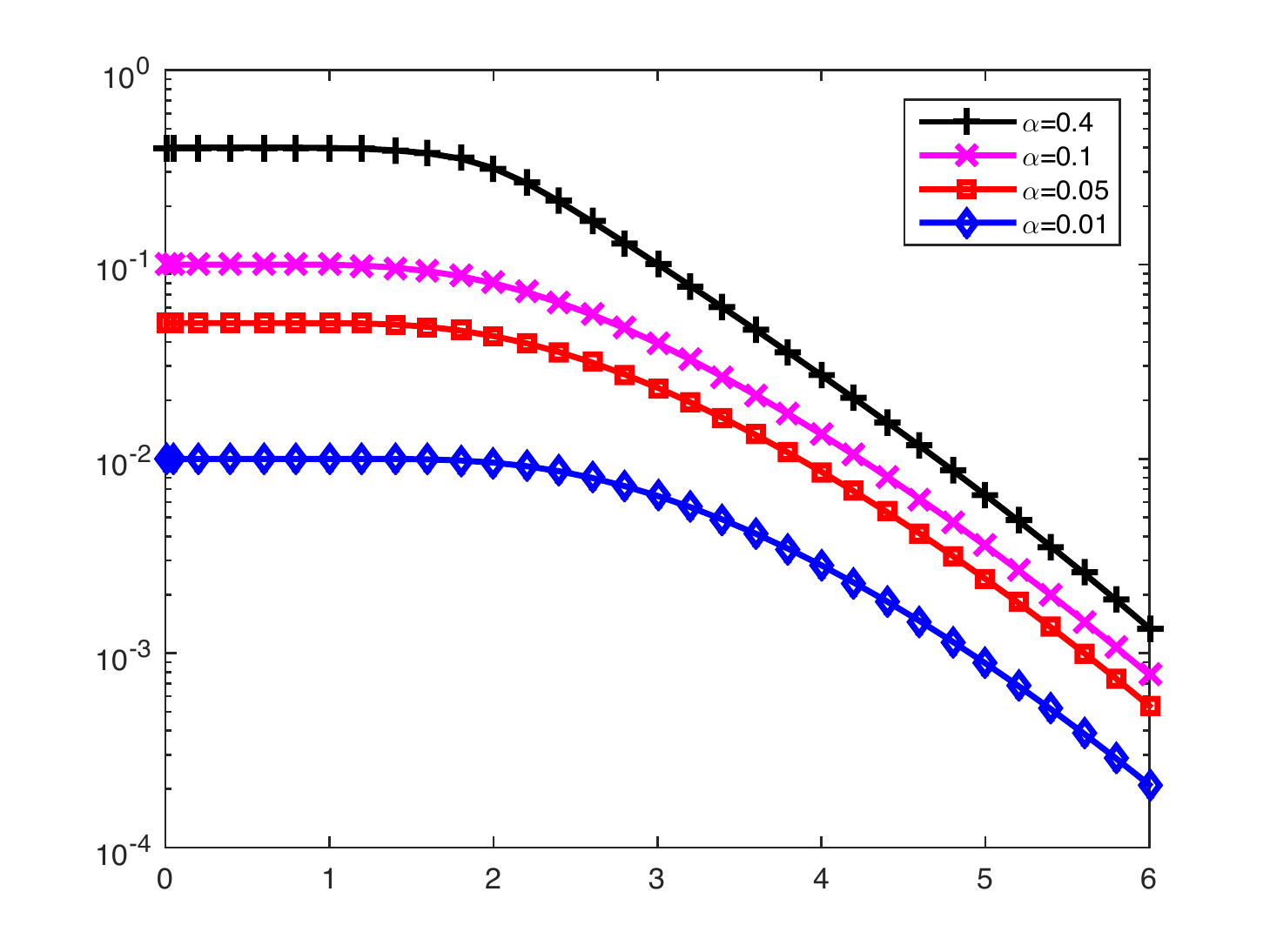}
\caption{Numerical calculation of $\expect{ Q \left(\frac{ v + U }{ \sqrt{v  } } \right) } $ for $v=\underline{v}=\overline{v} $ (y axis) versus $\mu$ (x axis) with different
$\alpha$. }
\label{fig:h_ErrorProb}
\end{figure}
\begin{figure}[ht]
\centering
\includegraphics[width=4in]{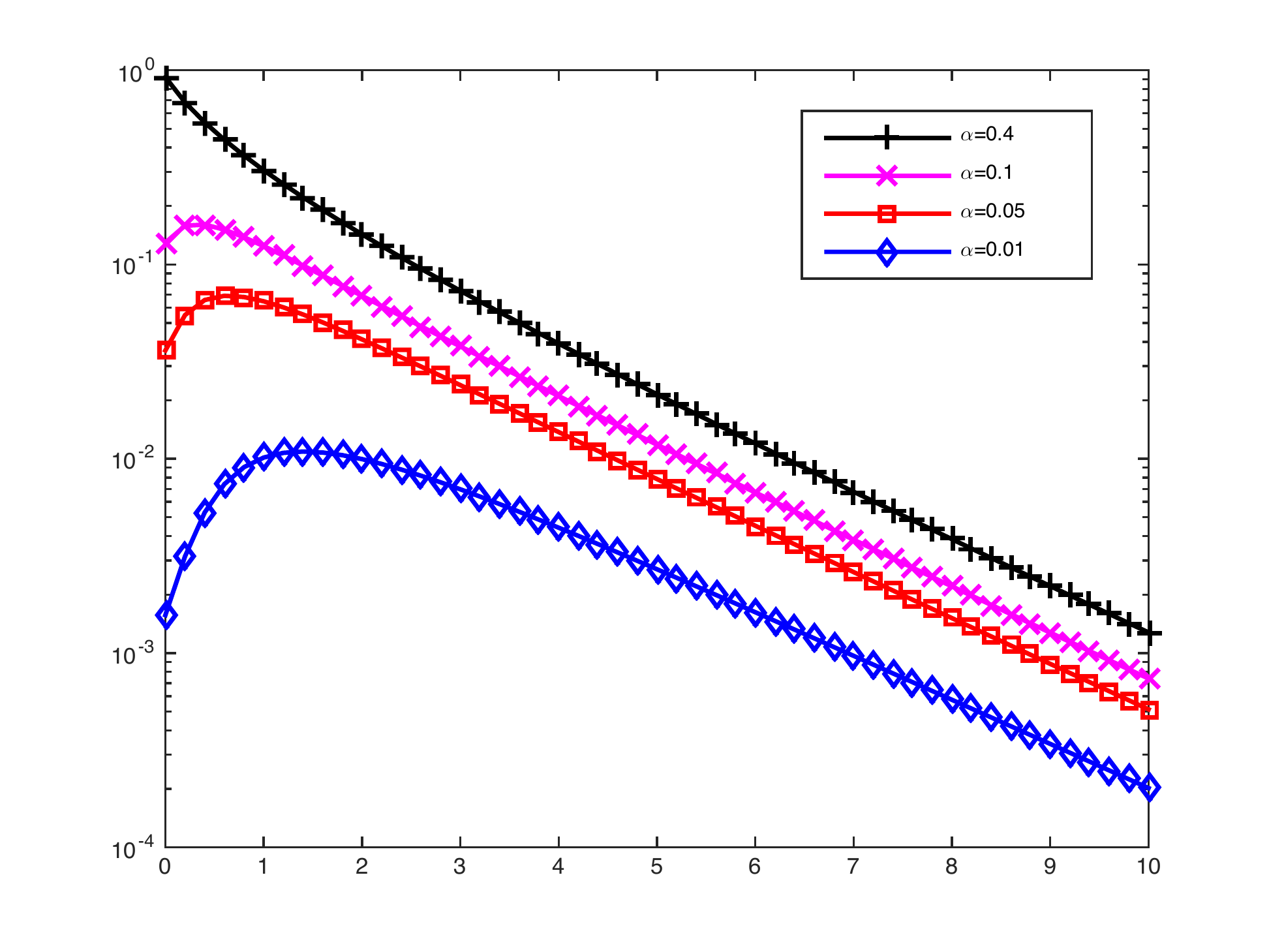}
\caption{Numerical calculation of $h'(v)$ (y axis) versus $v \in [0, 10]$ (x axis) with different $\alpha$.}
\label{fig:h_derivative_new}
\end{figure}

We are only able to show $h(v)$ has a unique fixed point if $|\mu|<2.$, but we believe that it is true for all $\mu$. 
\begin{conjecture}\label{conj:uniquenss}
For all $|\mu| \ge 2$ and $\alpha \in (0,1/2)$, $v=\mu^2 h(v)/4$ has a unique fixed point.
\end{conjecture}
It is tempting to prove Conjecture \ref{conj:uniquenss} by showing that $h(v)$ is concave in $v$ for all $\alpha \in (0,1/2)$.
However, through
numerical experiments depicted in Fig.~\ref{fig:h_derivative_new}, we find that $h(v)$ is convex around $v=0$ when $\alpha \le 0.1.$

In this work, we assume that there is a noisy label for every vertex in the graph. Previous work \cite{KaMoSc:14} instead assumes
that a fraction of vertices have true labels. 
However, in practice,
it is neither easy to get noisy labels for every vertex or true labels. 
Thus, there arises an interesting question: are local algorithms still optimal  with noisy labels available only for a small fraction of vertices? 

Moreover, we only studied the binary symmetric stochastic block model as a starting point. 
It would be of great interest to study to what extent our results generalize 
to the case with multiple clusters. Finally, the local algorithms are powerless in the symmetric stochastic block model simply 
because the local neighborhoods are statistically uncorrelated with the cluster structure. 
It is intriguing to investigate whether the local algorithms are optimal
when the clusters are of different sizes or connectivities.

\subsection{Paper outline}
The rest of the paper is organized as follows. We introduce the related inference problems on Galton-Watson tree model
in the next section, and show that the estimation accuracy of BP and the optimal estimation accuracy can be related to 
the estimation accuracy on the tree model. Section \ref{sec:uniqueness} shows the optimality of BP in the uniqueness regime $|a-b|<2$.
The proof of the optimality of BP in the high SNR regime $(a-b)^2>C(a+b)$ is presented in Section \ref{sec:highsnr}. 
Section \ref{sec:accurateSI} proves the optimality of BP when the side information is very  accurate. 
Finally, section \ref{sec:density evolution} characterizes the estimation accuracy of BP and the optimal estimation accuracy 
in the large degree regime based on density evolution analysis with Gaussian approximations.

\section{Inference problems on Galton-Watson tree model} \label{sec:tree model}

A key to understanding the inference problem on the graph is understanding the corresponding inference problem
on Galton-Watson trees. We introduce the problem now.
\begin{definition}
For a vertex $u$, we denote by $(T_u, \tau, \tilde{\tau})$ the following Poisson two-type branching process tree rooted at $u$, where $\tau$ is a $
\pm$ labeling of the vertices of $T$.
Let $\tau_u$ is chosen  from $\{\pm\}$ uniformly at random.
Now recursively for each vertex $i$ in $T_u$, given its label $\tau_i$, $i$ will have  $L_i \sim \Pois(a/2)$ children $j$ with
$\tau_j= + \tau_i$ and $M_i \sim \Pois(b/2)$ children $j$ with $\tau_j = -\tau_i$.
Finally for each vertex $i$,  let $\tilde{\tau}_i=\tau_i$ with probability $1-\alpha$
and $\tilde{\tau}_i= -\tau_i$ with probability $\alpha$.
\end{definition}

It follows that the distribution of $\tau$ conditional on $\tilde{\tau}$  and a finite $T_u$ is given by
\begin{align*}
\prob{ \tau | \tilde{\tau}, T_u } \propto \exp \left( \beta \sum_{i \sim j} \tau_i \tau_j + \sum_{i} h_i \tau_i \right),
\end{align*}
where $\beta=\frac{1}{2} \log \frac{a}{b}$, $h_i= \frac{1}{2} \tilde{\tau}_i \log \frac{1-\alpha}{\alpha}$, and $i \sim j$ means that $i$ and $j$ are connected in $T_u$. Observe that
$\prob{ \tau | \tilde{\tau}, T_u } $ is an Ising distribution on tree $T_u$ with external fields given by $h$.

Let $T_i^t$ denote the subtree of $T_u$ rooted at vertex $i$ of depth $t$, and $\partial T_i^t$ denote the set of vertices at the boundary of $T_i^t$.
With a bit abuse of notation, let $\tau_{A}$ denote the vector consisting of labels of vertices in $A$,
where $A$ could be either a set of vertices or a subgraph in $T_u$.
Similarly we define $\tilde{\tau}_{A}$.  We first consider the problem of estimating $\tau_u$ given observation of $ T_u^t,$ $\tau_{\partial T_u^t} ,$ and $\tilde{\tau}_{T_u^t }  $. Notice that the true labels of vertices in $T_u^{t-1}$ are not observed in this case.
\begin{definition}\label{def:treeexact}
The \emph{detection problem with side information in the tree and exact information at the boundary of the tree} is the inference problem
of inferring $\tau_u$ from the observation of $ T_u^t,$ $\tau_{\partial T_u^t} ,$ and $\tilde{\tau}_{T_u^t }  $. The success probability for an estimator
$\hat{\tau}_u(T_u^t, \tau_{\partial T_u^t} , \tilde{\tau}_{T_u^t }  )$ is defined by
\begin{align*}
p_{T^t} (\hat{\tau}_u) = \frac{1}{2} \prob{ \hat{\tau}_u=0 | \tau_u=0} + \frac{1}{2} \prob{ \hat{\tau}_u=1 | \tau_u=1}.
\end{align*}
Let $p_{T^t} ^\ast$ denote the optimal success probability.
\end{definition}
It is well-known that the optimal estimator in maximizing $p_T^t$, is the maximum a posterior (MAP) estimator. Since the prior distribution of $\tau_u$ is uniform over $\{\pm\}$, the MAP estimator
is the same as the maximum likelihood (ML) estimator, which can be expressed in terms of log likelihood ratio:
\begin{align*}
\hat{\tau}_{\rm ML} = 2 \times  \indc{\Lambda_u^t  \ge 0} -1,
\end{align*}
where
\begin{align*}
\Lambda^t_i \triangleq \frac{1}{2} \log \frac{ \prob{ T_i^t, \tau_{\partial T_i^t} , \tilde{\tau}_{T_i^t}   | \tau_i=+ }} { \prob{ T_i^t, \tau_{ \partial T_i^t} , \tilde{\tau}_{T_i^t}   |  \tau_i=-} },
\end{align*}
for all $i$ in $T_u$.
Moreover, the optimal success probability $p_{T^t} ^\ast$  is given by
\begin{align}
p_{T^t} ^\ast= \frac{1}{2} \expect{ |X_u^t| } + \frac{1}{2},  \label{eq:optimalestimationaccuracytree}
\end{align}
where $X_i^t$ is known as \emph{magnetization} given by
\begin{align*}
X_i^t  \triangleq \prob{ \tau_i=+ | T_i^t, \tau_{\partial T_i^t} ,  \tilde{\tau}_{T_i^t}  } - \prob{ \tau_i=- | T_i^t,  \tau_{\partial T_i^t } ,  \tilde{\tau}_{T_i^t}  }
\end{align*}
for all $i$ in $T_u$. In view of the identity $\tanh^{-1} (x) = \frac{1}{2} \log \left( \frac{1+x}{1-x}\right)$ for $x \in [-1,1]$, we have that $\tanh^{-1} (X_i^t) = \Lambda_i^t$.

We then consider the problem of estimating $\tau_u$ given observation of $ T_u^t$ and $\tilde{\tau}_{T_u^t }  $. Notice that in this case the true labels of vertices in $T_u^{t}$ are not observed.
\begin{definition}\label{def:treefree}
The \emph{detection problem with side information in the tree} is the inference problem
of inferring $\tau_u$ from the observation of $ T_u^t$ and $\tilde{\tau}_{T_u^t }  $. The success probability for an estimator
$\hat{\tau}_u (T_u^t, \tilde{\tau}_{T_u^t })$ is defined by
\begin{align*}
q_{T^t} (\hat{\tau}_u) = \frac{1}{2} \prob{ \hat{\tau}_u=0 | \tau_u=0} + \frac{1}{2} \prob{ \hat{\tau}_u=1 | \tau_u=1}.
\end{align*}
Let $q_{T^t}^\ast$ denote the optimal success probability.
\end{definition}
We remark that the only difference between \prettyref{def:treeexact} and \prettyref{def:treefree} is that  the exact labels at the boundary of the tree is revealed to  estimators in the former
and hidden in the latter. The optimal estimator in maximizing $q_{T^t}$ can be also expressed in terms of log likelihood ratio:
\begin{align*}
\hat{\tau}_{\rm ML} = 2 \times  \indc{\Gamma_u^t  \ge 0} -1,
\end{align*}
where
\begin{align*}
\Gamma^t_i \triangleq \frac{1}{2} \log \frac{ \prob{ T_i^t,  \tilde{\tau}_{T_i^t}   | \tau_i=+ }} { \prob{ T_i^t, \tilde{\tau}_{T_i^t}   |  \tau_i=-} },
\end{align*}
for all $i$ in $T_u$.
Moreover, the optimal success probability $q_{T^t} ^\ast$  is given by
\begin{align}
q_{T^t} ^\ast= \frac{1}{2} \expect{ |Y_u^t| } + \frac{1}{2},  \label{eq:optimalestimationaccuracytreefree}
\end{align}
where magnetization $Y_i^t$ is given by
\begin{align*}
Y_i^t  \triangleq \prob{ \tau_i=+ | T_i^t,  \tilde{\tau}_{T_i^t}  } - \prob{ \tau_i=- | T_i^t,  \tilde{\tau}_{T_i^t}  }
\end{align*}
for all $i$ in $T_u$. Again we have that $\tanh^{-1} (Y_i^t) = \Gamma_i^t$.

The log likelihood ratios and magnetizations can be computed via the  belief propagation algorithm.
The following lemma gives  recursive formula to compute $\Lambda_i^t$ ($\Gamma_i^t$) and $X_i^t$ ($Y_i^t$): no approximations are needed.
Notice that the  $\Lambda_i^t$ and $\Gamma_i^t$ satisfy the same recursion but with different initialization; similarly for
$X_i^t$ and $Y_i^t$.
Let $\partial i$ denote the set of children of vertex $i$.
\begin{lemma}\label{lmm:recursion}
 Define $ F(x) = \tanh^{-1} \left( \tanh(\beta) \tanh(x) \right)$. Then for $t\ge 1$,
\begin{align}
\Lambda_{i  } ^t & =  h_{i} + \sum_{\ell \in \partial i  } F( \Lambda_{\ell  }^{t-1} ) \label{eq:Likelihoodrecursion} \\
\Gamma_{i  } ^t & =  h_{i} + \sum_{\ell \in \partial i  } F( \Gamma_{\ell  }^{t-1} ) \label{eq:Likelihoodrecursionfree},
\end{align}
where $\Lambda_i^0 = \infty$ if $\tau_i=+$ and $\Lambda_i^0=-\infty$ if $\tau_i=-$; $\Gamma_i^0=\gamma$ if $\tau_i=+$ and $\Gamma_i^0=-\gamma$ if $\tau_i=-$.
It follows that  for $t \ge 1$.
\begin{align}
X_i^t & =\frac{\eexp^{h_i} \prod_{\ell \in \partial i} (1+ \theta X_{i }^{t-1} ) - \eexp^{-h_i}\prod_{\ell \in \partial i}  (1- \theta X_{i}^{t-1} ) }{\eexp^{h_i} \prod_{\ell \in \partial i} (1+ \theta X_{i }^{t-1} ) + \eexp^{-h_i}\prod_{\ell \in \partial i}  (1- \theta X_{i }^{t-1} ) }
\label{eq:Magentizationrecursion} \\
Y_i^t  &=\frac{\eexp^{h_i} \prod_{\ell \in \partial i} (1+ \theta Y_{i }^{t-1} ) - \eexp^{-h_i}\prod_{\ell \in \partial i}  (1- \theta Y_{i}^{t-1} ) }{\eexp^{h_i} \prod_{\ell \in \partial i} (1+ \theta Y_{i }^{t-1} ) + \eexp^{-h_i}\prod_{\ell \in \partial i}  (1- \theta Y_{i }^{t-1} ) },
\label{eq:Magentizationrecursionfree}
\end{align}
where $X_i^0=1$ if $\tau_i=+$ and $X_i^0=-1$ if $\tau_i=-$; $Y_i^0= 1-2\alpha$ if $\tau_i=+$ and $Y_i^0=2\alpha-1$ if $\tau_i=-$.
\end{lemma}

\begin{proof}
By definition, the claims for $\Lambda_i^0 $ ($\Gamma_i^0$) and $X_i^0$ ($Y_i^0$) hold.
We prove the claims for $\Lambda_u^t$ with $t\ge 1$;
the claims for $\Lambda_{i }^t$ for $i \neq u$ and $\Gamma_i^t$ follow similarly.

A key point is
to use the independent splitting property of the Poisson distribution to give an equivalent description of the numbers of children
with each label for any vertex in the tree.     Instead of separately generating the number of children of with each label,  we can first
generate the total number of children and then independently and randomly label each child. Specifically,
for every vertex $i$ in $T_u$, let $N_i$  denote the total number of its children and assume $N_i \sim \Pois(d)$ with $d=(a+b)/2$.
For each child $j\in \partial i,$ independently of everything else,
$\tau_j=\tau_i$ with probability $a/(a+b)$ and $\tau_j=- \tau_i$ with probability $b/(a+b)$.  With this view, the observation
of tree $T_u^t$ itself gives no information on the label of $u$. The side information $\tilde{\tau}_u$
and then the conditionally independent messages from those children provide information.   To be precise, we have that
\begin{align*}
\Lambda_u^t  &= \frac{1}{2} \log \frac{ \prob{ \tilde{\tau}_u | \tau_u =+ } \prod_{i \in \partial u} \prob{ T_i^{t-1}, \tau_{\partial T_i^{t-1}}, \tilde{\tau}_{T_i^{t-1}} | \tau_u=+ }  }{ \prob{ \tilde{\tau}_u | \tau_u =- } \prod_{i \in \partial u} \prob{ T_i^{t-1}, \tau_{\partial T_i^{t-1}} \tilde{\tau}_{T_i^{t-1}} | \tau_u=- } } \\
& =  h_u + \frac{1}{2} \sum_{i \in \partial u} \log \frac{  \sum_{x \in \{\pm \} } \prob{T_i^{t-1},  \tau_{\partial T_i^{t-1}}, \tilde{\tau}_{T_i^{t-1}} , \tau_i =x | \tau_u=+ }  } {   \sum_{x \in \{\pm \} }  \prob{ T_i^{t-1}, \tau_{\partial T_i^{t-1}} \tilde{\tau}_{T_i^{t-1}}, \tau_i = x | \tau_u=- } } \\
& =  h_u + \frac{1}{2} \sum_{i \in \partial u} \log \frac{  \sum_{x \in \{\pm \} } \prob{  \tau_i =x | \tau_u=+ } \prob{T_i^{t-1}, \tau_{\partial T_i^{t-1}}, \tilde{\tau}_{T_i^{t-1}} | \tau_i=x}   } {   \sum_{x \in \{\pm \} }  \prob{ \tau_i = x | \tau_u=- } \prob{ T_i^{t-1},  \tau_{\partial T_i^{t-1}} \tilde{\tau}_{T_i^{t-1}} | \tau_i =x } }  \\
& =  h_u + \frac{1}{2} \sum_{i \in \partial u} \log \frac{  a \prob{T_i^{t-1}, \tau_{\partial T_i^{t-1}}, \tilde{\tau}_{T_i^{t-1}} | \tau_i=+ }  +  b \prob{T_i^{t-1}, \tau_{\partial T_i^{t-1}}, \tilde{\tau}_{T_i^{t-1}} | \tau_i= - }  } {  b \prob{T_i^{t-1}, \tau_{\partial T_i^{t-1}}, \tilde{\tau}_{T_i^{t-1}} | \tau_i=+ }  +  a\prob{T_i^{t-1}, \tau_{\partial T_i^{t-1}}, \tilde{\tau}_{T_i^{t-1}} | \tau_i= - }  }  \\
& = h_u + \frac{1}{2} \sum_{i \in \partial u} \log \frac{  {\eexp^{2\beta+ 2 \Lambda_{i } ^{t-1} }  } + 1 } {  \eexp^{2 \Lambda_{i } ^{t-1} }  + \eexp^{2\beta}  } ,
\end{align*}
where the first equality holds because $T_u^t$ is independent of $\tau_u$, and conditional on $\tau_u$, $\tilde{\tau}_u$ and $( \tau_{\partial T_i^{t-1}}, \tilde{\tau}_{T_i^{t-1}} )$ for all $i \in \partial u$ are independent;
the second equality holds due to the fact that $\tilde{\tau_u}=\tau_u$ with probability $1-\alpha$ and $\tilde{\tau_u}=- \tau_u$ with probability $\alpha$;
the third equality follows because conditional on $\tau_i$, $\tau_u$ is independent of  $( \tau_{\partial T_i^{t-1}}, \tilde{\tau}_{T_i^{t-1}} )$; the fourth equality holds because
$\tau_i=\tau_u$ with probability $a/(a+b)$ and $\tau_i=-\tau_u$ with probability $b/(a+b)$; the last equality follows from the definition of $\beta$ and $\Lambda_i^t$.
By the definition of $\tanh(x)$ and the fact that $\tanh^{-1} (x) = \frac{1}{2} \log \left( \frac{1+x}{1-x}\right)$ for $x \in [-1,1]$, it follows that
\begin{align*}
\Lambda_u^t &=  h_u + \frac{1}{2} \sum_{i \in \partial u} \log \frac{ 1 + \tanh(\beta) \tanh(\Lambda_i^{t-1} ) } {  1 - \tanh(\beta) \tanh(\Lambda_i^{t-1} )  } \\
& = h_{u} + \sum_{ i \in \partial u}  \tanh^{-1} \left( \tanh(\beta)  \tanh(\Lambda_i^{t-1} ) \right).
\end{align*}
Finally, notice that  $X_u^t= \tanh( \Lambda_u^t)$ and thus
\begin{align*}
X_u^t &= \tanh \left(  h_{u} + \sum_{ i \in \partial u}  \tanh^{-1} \left( \tanh(\beta) X_i^{t-1} \right) \right) \\
&=\tanh \left(  h_u+ \frac{1}{2} \sum_{i \in \partial u} \log \frac{ 1 + \tanh(\beta) X_i^{t-1}  } {  1 - \tanh(\beta) X_i^{t-1} } \right) \\
& = \frac{\eexp^{h_u} \prod_{i \in \partial u} (1+ \theta X_{i }^{t-1} ) - \eexp^{-h_u}\prod_{i \in \partial u}  (1- \theta X_{i}^{t-1} ) }{\eexp^{h_u} \prod_{i \in \partial u} (1+ \theta X_{i }^{t-1} ) + \eexp^{-h_u}\prod_{i \in \partial u}  (1- \theta X_{i }^{t-1} ) }.
\end{align*}
\end{proof}
As a corollary of \prettyref{lmm:recursion}, $\Lambda_u$ is monotone with respect to the boundary conditions.
For any vertex $i$ in $T_u$, let
\begin{align*}
\Lambda^t_i (\xi) \triangleq \frac{1}{2} \log \frac{ \prob{ T_i^t, \tau_{\partial T_i^t} =\xi, \tilde{\tau}_{T_i^t}   | \tau_i=+ }} { \prob{ T_i^t, \tau_{ \partial T_i^t} =\xi , \tilde{\tau}_{T_i^t}   |  \tau_i=-} },
\end{align*}
where $\xi \in \{ \pm\}^{|\partial T_i^t |}$.
\begin{corollary}\label{cor:monotone}
Fix any vertex $i$ in $T_u$ and $t \ge 1$.
If the boundary conditions $\xi$ and $\hat{\xi }$ are such that
$\hat{\xi}_\ell \ge \xi_\ell$ for all $\ell \in \partial T_i^t$, then $\Lambda^t_i ( \hat{\xi} ) \ge \Lambda^t_i(\xi )$ for $a \ge b$
and $\Lambda^t_i ( \hat{\xi} ) \le \Lambda^t_i(\xi )$ otherwise.
In particular, $\Lambda^t_i(\xi)$ is maximized for $a \ge b$ and minimized for $a \le b$,
when $\xi_\ell= +$ for all $\ell \in \partial T_i^t$.
\end{corollary}
\begin{proof}
Recall that $ F(x) = \tanh^{-1} \left( \tanh(\beta) \tanh(x) \right)$. Then $F( \pm \infty)= \pm \beta$.
We prove the corollary by induction.
Suppose that $a \ge b$ and thus $\beta \ge 0$. We first check the base case $t=1$.
By definition, $\Lambda^0_i ( + ) = \infty$
and $\Lambda^0_i (-) = -\infty$. In view of \prettyref{eq:Likelihoodrecursion},
\begin{align*}
\Lambda_i^{1}(\hat{\xi}) - \Lambda_i^{1} (\xi) = \sum_{\ell \in \partial i} \left[ F \left(\Lambda_\ell^{0} (\hat{\xi}_{\ell} ) \right) - F\left(\Lambda_\ell^{0} ( \xi_{\ell} ) \right)  \right] = 2 \beta \sum_{\ell \in \partial i} \left( \indc{ \hat{\xi}_{\ell} > \xi_{\ell} } - \indc{ \hat{\xi}_{\ell} < \xi_{\ell} } \right) \ge 0.
\end{align*}
Suppose the claim holds for $t \ge 1$; we need to prove that the claim holds for $t+1$. In particular, in view of \prettyref{eq:Likelihoodrecursion},
\begin{align*}
\Lambda_i^{t+1}(\hat{\xi}) - \Lambda_i^{t+1} (\xi) = \sum_{\ell \in \partial i} \left[ F \left(\Lambda_\ell^{t} (\hat{\xi}_{\partial T_\ell^t}  ) \right) - F\left(\Lambda_\ell^{t} ( \xi_{\partial T_\ell^t} ) \right)  \right].
\end{align*}
Notice that
\begin{align*}
F'(x) = \frac{\tanh(\beta) (1-\tanh^2(x) ) }{1-\tanh^2(\beta)\tanh^2(x)}.
\end{align*}
Since $\beta \ge 0$ and thus $0 \le\tanh(\beta)<1$, it follows that $0 \le F'(x) \le \tanh(\beta)$. By the induction hypothesis, $\Lambda_\ell^{t} (\hat{\xi} ) \ge \Lambda_\ell^{t} ( \xi)$.
Hence, $\Lambda_i^{t+1}(\hat{\xi}) \ge \Lambda_i^{t+1}(\xi) $ and the corollary follows in case $a \ge b$. The proof for $a<b$ is the same except that
$\beta <0$; thus  $\Lambda_i^{1}(\hat{\xi}) \le \Lambda_i^{1} (\xi)$ and $ -\tanh (\beta) \le F'(x) \le 0$.
\end{proof}


\subsection{Connection between the graph problem and tree problems}

For the detection problem on graph, recall that $p_{G_n} (\hat{\sigma}_{\rm BP}^t ) $ denote the estimation accuracy of $\hat{\sigma}_{\rm BP}^t$ as per \prettyref{eq:estimationaccuracy};
$p_{G_n}^\ast$ is the optimal estimation accuracy.
For the detection problems on tree, recall that $p_{T^t}^\ast$ is the optimal estimation accuracy of estimating $\tau_u$ based on $T_u^t$, $\tilde{\tau}_{T_u^t}$, and
$\tau_{\partial T_u^t}$ as per \prettyref{eq:optimalestimationaccuracytree}; $q_{T^t}^\ast$ is the optimal estimation accuracy of estimating $\tau_u$ based on $T_u^t$ and $\tilde{\tau}_{T_u^t}$
as per \prettyref{eq:optimalestimationaccuracytreefree}.
In this section, we show that $  p_{G_n} (\hat{\sigma}_{\rm BP}^t ) $ equals to $q_{T^t}^\ast$ asymptotically, 
and $p_{G_n}^\ast$ is bounded by $ p_{T^t}^\ast$ from above for any $t\ge 1.$
Notice that the dependency of $q_{T^t}^\ast$ and $p_{T^t}^\ast$ on $n$ is only through the dependency of $a$ and $b$ on $n$.
Hence, if $a$ and $b$ are fixed constants, then both $q_{T^t}^\ast$ and $p_{T^t}^\ast$ do not depend on $n$.

A key ingredient  is to show $G$ is locally tree-like in the regime $a=n^{o(1)}$.  
Let $G_u^t$ denote the subgraph of $G$ induced by vertices whose distance to  $u$ is at most $t$ and let $\partial G_u^t$ denote the
set of vertices whose distance  from $u$ is precisely $t$. With a bit abuse of notation, let $\sigma_{A}$ denote the vector consisting of labels of vertices in $A$,
where $A$ could be either a set of vertices or a subgraph in $G$. Similarly we define $\tilde{\sigma}_A$.
The following lemma proved in \cite{MONESL:15} shows that we can construct a coupling such that
$(G^t, \sigma_{G^t}, \tilde{\sigma}_{G^t} ) = (T^t, \tau_{T^t}, \tilde{\tau}_{T^t} )$ with probability converging to $1$ when $a^{t} =n^{o(1)}.$
\begin{lemma}\label{lmm:couplingtree}
For $t =t(n)$ such that $a^{t} =n^{o(1)}$, 
there exists a coupling between $(G, \sigma, \tilde{\sigma} )$ and $(T, \tau, \tilde{\tau} )$
such that $(G^t, \sigma_{G^t}, \tilde{\sigma}_{G^t} ) = (T^t, \tau_{T^t}, \tilde{\tau}_{T^t} )$ with probability converging to $1$.
\end{lemma}


In the following, for ease of notation, we write  $T_u^t$ as $T^t$  and $G_u^t$ as $G^t$ when there is no ambiguity. 
Suppose that $(G^t, \sigma_{G^t}, \tilde{\sigma}_{G^t} ) = (T^t, \tau_{T^t}, \tilde{\tau}_{T^t} )$, then by comparing 
BP iterations \prettyref{eq:mp_commun} and \prettyref{eq:mp_combine_commun} with the recursions of log likelihood ratio $\Gamma^t$ \prettyref{eq:Likelihoodrecursionfree},
we find that $R_u^t$ exactly equals to $\Gamma_u^t$. In other words, when local neighborhood of $u$ is a tree, 
the BP algorithm defined in \prettyref{alg:MP_commun} exactly computes the log likelihood ratio $\Gamma_u^t$ for the tree model.
Building upon this intuition, the following lemma shows that $  p_{G_n} (\hat{\sigma}_{\rm BP}^t ) $ equals to $q_{T^t}^\ast$ asymptotically.

\begin{lemma}\label{lmm:optBPcondition}
For $t =t(n)$ such that $a^{t} =n^{o(1)}$,
\begin{align*}
\lim_{n \to \infty} | p_{G_n} (\hat{\sigma}_{\rm BP} ) - q_{T^t}^\ast | =0.
\end{align*}
\end{lemma}
\begin{proof}
In view of \prettyref{lmm:couplingtree}, we can construct a coupling such that
$(G^t, \sigma_{G^t}, \tilde{\sigma}_{G^t} ) = (T^t, \tau_{T^t}, \tilde{\tau}_{T^t} )$ with probability converging to $1$.
On the event $(G^t, \sigma_{G^t}, \tilde{\sigma}_{G^t} ) = (T^t, \tau_{T^t}, \tilde{\tau}_{T^t} )$, 
we have that  $R_u^t=\Gamma_u^t$. Hence,
\begin{align}
p_{G_n} (\hat{\sigma}_{\rm BP} )  = q_{T^t}^\ast +o(1),  \label{eq:BPaccuracy}
\end{align}
where $o(1)$ term comes from the coupling error.
\end{proof}

We are going to  show that as $n \to \infty$, $ p_{G_n}^\ast $  is bounded by $ p_{T^t}^\ast$ from above for any $t\ge 1.$
Before that, we need a key lemma which shows that conditional on $(G^t,  \tilde{\sigma}_{G^t}, \sigma_{\partial G^t })$, $\sigma_u$ is almost independent of the graph structure and noisy labels outside of $G^t$.

\begin{lemma}\label{lmm:asymptoticalIndependence}
For $t =t(n)$ such that $a^{t} =n^{o(1)}$,  there exists a sequence of events $\calE_n$ such that $\prob{\calE_n} \to 1$ as $n \to \infty$, and 
on event $\calE_n$,
\begin{align}
\prob{ \sigma_ u  =x | G^t,  \tilde{\sigma}_{G^t}, \sigma_{\partial G^t } } = (1+o(1))  \prob{ \sigma_ u =x | G, \tilde{\sigma}, \sigma_{\partial G^t} }, \quad \forall x \in\{\pm\}.
\label{eq:asymInd}
\end{align}
Moreover, on event $\calE_n$,  $(G^t, \sigma_{G^t}, \tilde{\sigma}_{G^t} ) = (T^t, \tau_{T^t}, \tilde{\tau}_{T^t} )$ holds. 
\end{lemma}
\begin{proof}
Recall that $G^t$ is the subgraph of $G$ induced by vertices whose distance from $u$ is at most $t$.
Let $A^t$ denote the set of vertices in $G^{t-1}$, $B^t$ denote the set of vertices in $G^{t}$,
and $C^t$ denote the set of vertices in $G$ but not in $G^t$. Then $A^t \cup \partial G^t=B^t$ and $A^t \cup \partial G^t \cup C^t = V.$
Define $s_{A} =\sum_{i \in A^t } \sigma_i$ and
 $s_{C} = \sum_{i \in C^t} \sigma_i$. 
Let 
\begin{align*}
\calE_n=\{ (\sigma_C, G^t): | s_C | \le n^{0.6},  |B| \le n^{0.1} , (G^t, \sigma_{G^t}, \tilde{\sigma}_{G^t} ) = (T^t, \tau_{T^t}, \tilde{\tau}_{T^t} )\}.
\end{align*} 
Next, we show  $\prob{\calE_n} \to 1.$  

By the assumption $a^{t}=n^{o(1)}$, it follows that $(G^t, \sigma_{G^t}, \tilde{\sigma}_{G^t} ) = (T^t, \tau_{T^t}, \tilde{\tau}_{T^t} )$
and $|B|= n^{o(1)} $ with probability converging to $1$ (see \cite[Proposition 4.2]{MONESL:15} for a formal proof).
Note that $s_C = 2 X- |C|$ for some $X \sim \Binom(|C|, 1/2)$. Letting $\alpha_n=n^{0.6}$, in view of the Bernstein inequality, 
 \begin{align*}
 \prob{|s_C|> \alpha_n } = \prob{  \big| X - |C|/2 \big| > \alpha_n/2 } \le 2 \eexp^{ - \frac{-\alpha_n^2/4}{ |C|/2 + \alpha_n/3 } } =o(1),
 \end{align*}
 where the last equality holds because $ |C| \le n$ and $\alpha_n / \sqrt{n} \to \infty$.
In conclusion, we have that $\prob{\calE_n} \to 1$ as $n \to \infty$. 

Finally, we prove that \prettyref{eq:asymInd} holds. 
We start by proving that on event $\calE_n$,
\begin{align}
\prob{ \sigma_ u  =x, G^t,  \tilde{\sigma}_{G^t}, \sigma_{\partial G^t } }  = (1+o(1))  K \; \prob{ \sigma_ u =x,  G^t, \tilde{\sigma}_{G^t}, \sigma_{\partial G^t }, \sigma_C  }  ,   \quad \forall x \in\{\pm\}.  \label{eq:desiredInd}
\end{align}
for some $K$ which does not depend on $x$. To proceed, we write the joint distribution of $\sigma$, $G$, and  $\tilde{\sigma}$ as a product form.

For any two sets $U_1, U_2 \subset V$,
define
\begin{align*}
\Phi_{U_1, U_2} (G, \sigma) = \prod_{ (u,v) \in U_1 \times U_2 } \phi_{uv} (G, \sigma),
\end{align*}
where $(u,v)$ denotes an unordered pair of vertices and
\begin{align}
\phi_{uv}(G,L,\sigma)=\left \{
\begin{array}{rl}
  a/n & \text{if } \sigma_u=\sigma_v, (u,v) \in E \\
 b/n & \text{if } \sigma_u \neq \sigma_v, (u,v) \in E\\
 1-a/n & \text{if } \sigma_u= \sigma_v, (u,v) \notin E\\
 1-b/n & \text{if } \sigma_u \neq \sigma_v, (u,v) \notin E \nonumber
\end{array} \right.
\end{align}
For any set $U \subset V$, define $ \Phi_{U} (\tilde{\sigma},\sigma ) = \prod_{ u  \in U }   \phi_{u} (\tilde{\sigma}_u, \sigma_u)$,
where
\begin{align*}
\phi_{u} (\tilde{\sigma}_u, \sigma_u ) \left \{
\begin{array}{rl}
  1-\alpha & \text{if } \tilde{\sigma}_u=\sigma_u \\
\alpha & \text{if }  \tilde{\sigma}_u \neq \sigma_u  \nonumber
\end{array} \right.
\end{align*}
Then the joint distribution of $\sigma$, $G$, and  $\tilde{\sigma}$ can be written as
\begin{align*}
\prob{ \sigma, G, \tilde{\sigma} } = 2^{-n} \Phi_B \; \Phi_C \; \Phi_{B, B} \; \Phi_{C, C} \; \Phi_{\partial G^t, C} \; \Phi_{A, C}.
\end{align*}
Notice that $A$ and $C$ are disconnected.
We claim that on event $\calE_n$, $\Phi_{ A, C} $ only depend on $\sigma$ through the $o(1)$ term. In particular, on event $\calE_n$, 
\begin{align*}
 \Phi_{A, C} (G, \sigma  ) & = \prod_{ (u,v) \in A \times C } \phi_{uv} (G, \sigma ) \\
& = \left( 1- \frac{a}{n} \right)^{(|A| |C| + s_A s_C ) /2} \left( 1- \frac{b}{n} \right)^{(|A| |C| - s_A s_C ) /2} \\
& = (1+o(1)) \left( 1- \frac{a}{n} \right)^{ |A| |C|  /2 } \left( 1- \frac{b}{n} \right)^{ |A| |C| /2} \\
& \triangleq (1+o(1)) K(|A|, |C| ),
\end{align*}
where $K(|A|,|C|)$ only depends on $|A|$ and $|C|$; the second equality holds because $ u \in A$ and $v \in C$ implies that $(u,v) \notin E$ and thus
$\phi_{uv}$ is either $1-a/n$ or $1-b/n$; the third equality holds because $|s_A s_C| \le |B| |s_C| \le  n^{0.7}=o(n)$. As a consequence, 
\begin{align*}
\prob{ \sigma, G, \tilde{\sigma}, \calE_n} = (1+o(1)) 2^{-n} \; K(|A|, |C| )  \;   \Phi_B \; \Phi_C \; \Phi_{B, B} \; \Phi_{C, C} \; \Phi_{\partial G^t, C} , 
\end{align*}
It follows that 
\begin{align*}
\prob{ \sigma_ u =x,  G^t, \tilde{\sigma}_{G^t}, \sigma_{\partial G^t }, \sigma_C, \calE_n  } = (1+o(1)) 2^{-n} \; K(|A|, |C| )  \;  \sum_{ \sigma_{A\backslash\{u\} } } \Phi_B \;  \Phi_{B, B } \triangleq (1+o(1)) K',
\end{align*}
where $K'$ does not depend on $\sigma_C$. 
Hence, 
\begin{align*}
\prob{ \sigma_ u =x,  G^t, \tilde{\sigma}_{G^t}, \sigma_{\partial G^t },\calE_n}   & =  \sum_{\sigma_C} \prob{ \sigma_ u =x,  G^t, \tilde{\sigma}_{G^t}, \sigma_{\partial G^t }, \sigma_C, \calE_n }  \\
& = (1+o(1)) K'  \; \sum_{\sigma_C} \indc{|s_C| \le n^{0.6} } = (1+o(1))  K' 2^{|C|} \prob{ |s_C| \le n^{0.6} },
\end{align*}
and the desired \prettyref{eq:desiredInd} follows with $K= 2^{|C|} \prob{ |s_C| \le n^{0.6} }.$
By Bayes' rule, it follows from \prettyref{eq:desiredInd} that on event $\calE_n$, 
\begin{align*}
\prob{ \sigma_ u  =x | G^t,  \tilde{\sigma}_{G^t}, \sigma_{\partial G^t } } =  (1+o(1))  \prob{ \sigma_ u =x | G^t, \tilde{\sigma}_{G^t}, \sigma_{\partial G^t }, \sigma_C }, \quad \forall x \in\{\pm\}.
\end{align*}
Hence, on event $\calE_n,$
\begin{align*}
\prob{ \sigma_ u  =x | G,  \tilde{\sigma}, \sigma_{\partial G^t }  }  & = \sum_{\sigma_C } \prob{ \sigma_ u  =x, \sigma_C | G,  \tilde{\sigma}, \sigma_{\partial G^t }  } \\
& = \sum_{\sigma_C } \prob{ \sigma_C | G,  \tilde{\sigma}, \sigma_{\partial G^t } } \prob{\sigma_u=x | G,  \tilde{\sigma}, \sigma_{\partial G^t } , \sigma_C } \\
& = \sum_{\sigma_C} \prob{ \sigma_C | G,  \tilde{\sigma}, \sigma_{\partial G^t } } \prob{\sigma_u=x | G^t,  \tilde{\sigma}_{G^t}, \sigma_{\partial G^t } , \sigma_C } \\
& =  (1+o(1))  \prob{ \sigma_ u  =x | G^t,  \tilde{\sigma}_{G^t}, \sigma_{\partial G^t }},
\end{align*}
where the third equality holds, because conditional on $(G^t,  \tilde{\sigma}_{G^t}, \sigma_{\partial G^t }, \sigma_C)$,
$\sigma_u$ is independent of the graph structure and noisy labels outside of $G^t$;
the last equality follows due to the last displayed equation. Hence, on the event $\calE_n$, the desired  \prettyref{eq:asymInd} holds.

\end{proof}
\begin{lemma}\label{lmm:accuracyupperbound}
For $t =t(n)$ such that $a^{t} =n^{o(1)}$,  
\begin{align*}
\limsup_{n \to \infty} ( p^\ast_{G_n} - p^\ast_{T^t}  ) \le 0.
\end{align*}
\end{lemma}
\begin{proof}
In view of \prettyref{eq:optimalaccuracygraph},
\begin{align*}
p_{G_n} ^\ast =  \frac{1}{2} \expect{ \big| \prob{\sigma_u=+  | G, \tilde{\sigma} } -  \prob{\sigma_u=-  |G, \tilde{\sigma} } \big| }+ \frac{1}{2}.
\end{align*}
Consider estimating $\sigma_u$ based on $G$ and $ \tilde{\sigma}$.
For a fixed $t \in \naturals$,  suppose a genie reveals the labels of all vertices whose distance from $u$ is precisely $t$, and
let $\hat{\sigma}_{ {\rm Oracle}, t}$ denote the optimal oracle estimator given by
\begin{align*}
\hat{\sigma}_{ {\rm Oracle} ,t} (u) =2 \times \indc{ \prob{\sigma_u =+ | G, \tilde{\sigma},  \sigma_{ \partial G^t} } \ge \prob{\sigma_u =- |  G, \tilde{\sigma}, \sigma_{\partial G^t } } } -1.
\end{align*}
Let $p_{G_n } (\hat{\sigma}_{ {\rm Oracle},t})  $ denote the success probability  of the oracle estimator, which is given by
\begin{align*}
p_{G_n } (\hat{\sigma}_{ {\rm Oracle},t} )   = \frac{1}{2}  \expect{ \big| \prob{\sigma_u =+ | G, \tilde{\sigma}, \sigma_{ \partial G^t}} - \prob{\sigma_u =- | G, \tilde{\sigma},  \sigma_{\partial G^t } }  \big| } + \frac{1}{2}.
\end{align*}
Since $\hat{\sigma}_{{\rm Oracle}, t}(u)$ is optimal with the extra information $ \sigma_{\partial G^t }$,
it follows that $p_{G_n } (\hat{\sigma}_{ {\rm Oracle} ,t } ) \ge p^\ast_{G_n}$ for all $t$ and $n$. 
\prettyref{lmm:asymptoticalIndependence} implies
that there exists a sequence of events $\calE_n$ such that $\prob{\calE_n \to 1}$ and on event $\calE_n$, 
\begin{align*}
\prob{ \sigma_ u  =x | G^t,  \tilde{\sigma}_{G^t}, \sigma_{\partial G^t } } = (1+o(1))  \prob{ \sigma_ u =x | G, \tilde{\sigma}, \sigma_{\partial G^t}}, \quad \forall x \in\{\pm\},
\end{align*}
and $(G^t, \sigma_{G^t}, \tilde{\sigma}_{G^t} ) = (T^t, \tau_{T^t}, \tilde{\tau}_{T^t} )$ holds. 
It follows that 
\begin{align*}
p_{G_n } (\hat{\sigma}_{ {\rm Oracle},t} )   
& = \frac{1}{2}  \expect{ \big| \prob{\sigma_u =+ | G, \tilde{\sigma}, \sigma_{ \partial G^t}} - \prob{\sigma_u =- | G, \tilde{\sigma},  \sigma_{\partial G^t } }  \big|  \indc{\calE_n} } + \frac{1}{2} + o(1) \\
& = \frac{1}{2}   \expect{ \big| \prob{\sigma_u =+ | G^{t}, \tilde{\sigma}_{G^{t} }, \sigma_{ \partial G^t} } - \prob{\sigma_u =- | G^{t}, \tilde{\sigma}_{G^{t} },  \sigma_{\partial G^t } }  \big|   \indc{\calE_n} } + \frac{1}{2} + o(1) \\
& = \frac{1}{2}   \expect{ \big|  \prob{ \tau_u=+  |  T^{t},  \tilde{\tau}_{ T^t} , \tau_{\partial T^t} }  -  \prob{ \tau_u=-  |  T^{t},  \tilde{\tau}_{ T^t}, \tau_{\partial T^t} } \big|   \indc{\calE_n} } + \frac{1}{2} + o(1) \\
& =  \frac{1}{2}   \expect{ \big|  \prob{ \tau_u=+  | T^{t},    \tilde{\tau}_{ T^t},  \tau_{\partial T^t} }  -  \prob{ \tau_u=-  |T^{t},   \tilde{\tau}_{ T^t} , \tau_{\partial T^t}  } \big| } + \frac{1}{2} + o(1) \\
&= p_{T^t}^\ast  + o(1). 
\end{align*}
Hence,
\begin{align*}
 p^\ast_{G_n} - p_{T^t}^\ast \le p_{G_n } (\hat{\sigma}_{ {\rm Oracle} ,t } ) - p_{T^t}^\ast \to 0.
\end{align*}
\end{proof}

The following is a simple corollary of \prettyref{lmm:optBPcondition} and \prettyref{lmm:accuracyupperbound}.
\begin{corollary}
For $t =t(n)$ such that $a^{t} =n^{o(1)}$,  
\begin{align*}
\limsup_{n \to \infty} \left( p^\ast_{G_n} -   p_{G_n} (\hat{\sigma}_{\rm BP} )  \right) \le \limsup_{n \to \infty} \left( p^\ast_{T^t} - q^\ast_{T^t} \right).
\end{align*}
\end{corollary}
The above corollary implies that  $\hat{\sigma}^t_{\rm BP}$  achieves the
 optimal estimation accuracy $p^\ast_{G_n}$ asymptotically, provided that $p^\ast_{T^t} - q^\ast_{T^t}$ converges to $0$, or
 equivalently, $\expect{|X_u^t-Y_u^t|}$ converges to $0$.
Notice that the only difference between $p^\ast_{T^t}$   and $q^\ast_{T^t}$ is that in the former, exact label information is revealed
at the boundary of $T^t$, while in the latter, only noisy label information at level $t$ is available.
In the next three sections, we provide three different sufficient conditions under which $\expect{|X_u^t-Y_u^t|}$ converges to $0$.

\section{Optimality of local BP when $|a-b|<2$ }\label{sec:uniqueness}

Recall that $\Lambda^t_u$ is a function of $\tau_{\partial T_u^t}$, \ie, the labels of vertices at the boundary of $T_u^t$.
Let $\Lambda^t_u (+)$ denote $\Lambda^t_u$ with the labels of vertices at the boundary of $T_u^t$ all equal to $+$.
Similarly define $\Lambda^t_u (-)$.
The following lemma shows that $\Lambda^t_u$ is asymptotically independent of  $\tau_{\partial T_u^t}$
as $t \to \infty$ if $|a-b|<2$.
\begin{lemma}\label{lmm:uniqueness}
 For all $ t \ge 0$, let $e(t+1) = \expect{ | \Lambda^{t+1}_u(+)- \Lambda^{t+1}_u(-)  | }$. Then
\begin{align*}
e(t+1) \le e(1) \left( |\tanh(\beta)| d \right)^{t},
\end{align*}
where $e(1)= 2\beta d $ and $d=(a+b)/2$. 
\end{lemma}
\begin{proof}
Recall that $ |F'(x)| \le |\tanh(\beta)|$. It follows that
\begin{align*}
\big| \Lambda^{t+1}_u(+)- \Lambda^{t+1}_u(-) \big| \le  \sum_{i \in \partial u }  \big| \left( F (\Lambda_{i}^t(+) ) - F(\Lambda_{i}^t (-1) ) \right)  \big|
\le  |\tanh(\beta)  | \sum_{i \in \partial u}   \big| \left( \Lambda_{i }^t(+) - \Lambda_{i }^t (-1) \right) \big|.
\end{align*}
Notice that the subtree $T_i^t$ has the same distribution as $T_u^t$. Thus $\Lambda_{i }^t$ has the same distribution
as $\Lambda^t_u$. Moreover, the number of children of root $u$ is independent of $\Lambda_{i }^t$ for all $i \in \partial u$.
Thus, taking the expectation over both hand sides of the last displayed equation, we get that
\begin{align*}
e(t+1)=  \expect{ \big| \Lambda^{t+1}_u(+)- \Lambda^{t+1}_u(-)  \big| } \le | \tanh(\beta) | d  \; \expect{ \big| \Lambda_{u}^t(+)- \Lambda_{u}^t (-) \big|  }
=  | \tanh(\beta) | d \;  e(t).
\end{align*}
Moreover,
\begin{align*}
e(1) = d \left( F(\infty) - F(-\infty)  \right) = 2 \beta d.
\end{align*}
\end{proof}
\begin{theorem}
If $|a-b|<2$, then  $  \lim_{t \to \infty}  \limsup_{n \to \infty} \expect{ |X_u^t -Y_u^t | } =0 $.
\end{theorem}
\begin{proof}
In view of \prettyref{lmm:uniqueness}, we have that
\begin{align*}
\expect{ \left| \Lambda^t_u(+)- \Lambda^t_u(-)  \right|} \le 2 \beta d  \left( \tanh(\beta) d \right)^{t-1}  = \left| \log \frac{a}{b} \right|  \frac{a+b}{2} \left( \frac{|a-b|}{2} \right)^{t-1},
\end{align*}
where we used the fact that $d=(a+b)/2$ and $\beta=  \log (a/b) /2$.
It follows from  \prettyref{cor:monotone} and the facts that $X^t_u=\tanh(\Lambda^t_u)$ and $Y^t_u=\tanh(\Gamma^t_u)$ that
\begin{align*}
| X_u^t - Y_u^t| \le | \Lambda_u^t- \Gamma_u^t| \le | \Lambda^t_u(+)- \Lambda^t_u(-)|.
\end{align*}
Combing the last two displayed equations gives
\begin{align*}
\expect{ |X_u^t -Y_u^t | } \le \left| \log \frac{a}{b} \right|  \frac{a+b}{2} \left( \frac{|a-b|}{2} \right)^{t-1}.
\end{align*}
By the assumption, $a-b<2$ for all sufficiently large $n$, and thus $\lim_{t\to \infty} \limsup_{n \to \infty} \expect{ |X_u^t -Y_u^t | }=0.$
\end{proof}

\section{Optimality of  local BP when $(a-b)^2>C(a+b)$} \label{sec:highsnr}
Recall that $d=(a+b)/2$ and $\beta= \frac{1}{2} \log \frac{a}{b}.$ We introduce the notation $\theta=\tanh(\beta)$ and $\eta=(1-\theta)/2$. 
Let $\expectp{X}$ and $\expectm{X}$ denote the expectation of $X$ conditional on
$\tau_u=+$ and $\tau_u=-$, respectively.
\begin{theorem}\label{thm:main}
There exists a constant $C$ depending only on $c_0$ such that if $(a-b)^2 \ge C(a+b)$, then
$\lim_{t \to \infty} \limsup_{n \to \infty}  \expect{ |X_u^t -Y_u^t | }=0.$
\end{theorem}

The proof of \prettyref{thm:main} is divided into three steps. Firstly, we show that when $\theta^2 d  = \frac{(a-b)^2}{2(a+b)}$ is large, then $\expectp{X_{u}^t }$ and $\expectp{Y_{u}^t}$
are close to $1$ for all sufficiently large $t$. This result allows us to analyze the recursions \prettyref{eq:Magentizationrecursion} and \prettyref{eq:Magentizationrecursionfree} by
assuming $\expectp{X_{u}^t }$ and $\expectp{Y_{u}^t}$ are close to $1$. Secondly, we study the recursions of $\expect{( X_u^t -Y_u^t )^2 }$ when  $|\theta|$ is small. Finally, we analyze the recursions 
of $\expect{\sqrt{|X_u^t -Y_u^t | } }$ when $|\theta|$ is large.
The partition of analysis of recursions into  small $|\theta|$ and large $|\theta|$ cases, and the study of different moments of $|X_u^t -Y_u^t |$, 
are related to the fact that for different values of $\theta$ we expect the distributions of $X_u^t$ and $Y_u^t$  correspond to different power-laws.
When $\theta$ is small, we have many small contributions from neighbors and therefore it is expected that $X_u^t$ and  $Y_u^t|$ will have  thin tails.
When $\theta$ is large, we have a few large contributions  from neighbors and we therefore expect fat tails.

\subsection{Large expected magnetization}
We first introduce a useful lemma to relate $ \expectp{X_{u}^t }$ with $\expect{|X_u^t|}$.
\begin{lemma}\label{lmm:expectedmagnetization}
\begin{align}
\expectp{ X_u^t } &\ge 2\expect{ |X_u^t |}  - 1, \nonumber \\
\expectp{Y_u^t } & \ge 2\expect{ |Y_u^t |}  - 1. \label{eq:magbound}
\end{align}
\end{lemma}
\begin{proof}
We prove the claim for $X_u^t$; the claim for $Y_u^t$ follows analogously. Observe that
\begin{align}
\expectp{ X_u^t } &=   \expectp{ X_u^t \indc{X_u^t \ge 0 }} +  \expectp{ X_u^t \indc{X_u^t < 0 } } \nonumber  \\
& = \expectp{ |X_u^t |  \indc{X_u^t \ge 0 }}  - \expectp{ |X_u^t | \indc{X_u^t < 0 }} \nonumber \\
& = \expectp{ |X_u^t |  } - 2\expectp{ |X_u^t | \indc{X_u^t < 0 } } \nonumber \\
& \ge  \expectp{ |X_u^t | }  -2  \prob{ X_u^t <0 \big| \tau_u=+}, \nonumber \\
& = \expect{ |X_u^t | }  -2  \prob{ X_u^t <0 \big| \tau_u=+}  \label{eq:maglarge}
\end{align}
where the last inequality holds because $|X_u^t|\le 1$; the last equality holds due to $\expectp{ |X_u^t |} = \expect{ |X_u^t |  }$,
Moreover, by definition,
\begin{align*}
\frac{1+ \expect{ |X_u^t |  } }{2} =\frac{1}{2} \prob{ X_u^t \ge 0 | \tau_u=+} + \frac{1}{2} \prob{ X_u^t < 0 | \tau_u=-}  \le  \prob{ X_u^t  \ge 0 | \tau_u=+},
\end{align*}
where the last inequality holds because by symmetry,  $\prob{ X_u^t < 0 | \tau_u=-} =  \prob{ X_u^t > 0 | \tau_u= +}$.
Thus $   \prob{ X_u^t <0 | \tau_u=+ } \le \frac { 1- \expect{ |X_u^t | } }{2}$ and it follows from \prettyref{eq:maglarge} that
\begin{align}
\mathbb{E}^{+} [ X_u^t ] \ge  \expect{ |X_u^t | } -  \left( 1- \expect{ |X_u^t | }  \right) = 2\expect{ |X_u^t |}  - 1.
\end{align}
\end{proof}

\begin{lemma}\label{lmm:magnetization}
There is a universal constant $C'>0$ and $t^\ast(\theta, d, \alpha)$ such that for all $t \ge t^\ast (\theta, d, \alpha)$,
\begin{align*}
 \expectp{X_{u}^t }  &\ge 1 - \frac{2C' \eta }{\theta^2 d} - 2 \eexp^{-C' d }, \\
 \expectp{Y_{u}^t } & \ge 1 - \frac{2C' \eta }{\theta^2 d} - 2 \eexp^{-C' d }.
\end{align*}
\end{lemma}
The proof of the lemma follows quite easily from a  similar statement in \cite{MNS:2013a} where the authors considered models where less information is provided - i.e., there is only information on the leaves of the trees (which are either noisy or not).
\begin{proof}
Notice that given access to $\tau_{\partial T_u^t}$ and $\tilde{\tau}_{T_u^t}$, the optimal estimator of $\tau_u$ is $\sgn(X_u^t)$ whose success probability
is  $(1+\expect{|X_u^t |} )/2$. Define
$
\tilde{X}_{u}^t = \prob{ \tau_u = + | \tau_{\partial T_u^t} }- \prob{ \tau_u=- | \tau_{\partial T_u^t} }.
$
It is shown in \cite{MNS:2013a} that there exists a universal constant $C'>0$  and a $t^\ast(\theta, d)$ such that for $ t \ge t^\ast$,
\begin{align*}
 \expect{ |\tilde{X}_u^t | } \ge 1- \frac{C' \eta }{\theta^2 d} - \eexp^{-C' d }.
\end{align*}
Consider the estimator $\sgn(\tilde{X}_u^t)$, whose success probability is  given by $\left(1+\expect{ | \tilde{X}_u^t | } \right)/2$.
Since $\sgn(X_u^t)$ is the optimal estimator, it follows that
\begin{align*}
\expect{|X_u^t |} \ge  \expect{ |\tilde{X}_u^t | } \ge 1- \frac{C' \eta }{\theta^2 d} - \eexp^{-C' d }.
\end{align*}
Hence, by \prettyref{lmm:expectedmagnetization},
\begin{align}
\expect{X_u^t} \ge 2\expect{ |X_u^t |}  - 1 \ge 1- \frac{ 2C' \eta }{\theta^2 d} - 2 \eexp^{-C' d }.
\label{eq:magbound}
\end{align}
The claim for $\mathbb{E}^{+} [ Y_{u}^t ]$ can be proved similarly except that we need to assume $t \ge t^\ast(\theta, d, \alpha)$ for some $t^\ast(\theta, d, \alpha)$ .

\end{proof}

\subsection{Small $|\theta|$ regime}

In this subsection, we focus on the regime $|\theta| \le \theta^\ast$, where $\theta^\ast$
is a small constant to be specified later. In this regime, to prove \prettyref{thm:main},
it is sufficient to show the following proposition.
\begin{proposition}\label{prop:smalltheta}
There exist universal constants $\theta^\ast>0$ and $C>0$ such that  if $ |\theta| \le \theta^\ast$ and $(a-b)^2 \ge C(a+b)$, then
for all $t \ge t^\ast(\theta, d, \alpha),$
\begin{align*}
\expect{(X_u^{t+1} -Y_{u}^{t+1} )^2  } \le \sqrt{\alpha (1-\alpha)}  \expect{(X_u^t - Y_u^t )^2}.
\end{align*}
\end{proposition}
\begin{proof}
Note that $ \expectp{(X_u^t - Y_u^t )^2} = \expectm{(X_u^t - Y_u^t )^2} =  \expect{(X_u^t - Y_u^t )^2}$. Hence, it suffices to show that
\begin{align}
\expectp{(X_u^{t+1} -Y_{u}^{t+1} )^2  } \le \sqrt{\alpha (1-\alpha)}  \expectp{(X_u^t - Y_u^t )^2}. \label{eq:desiredsmalltheta}
\end{align}
Fix $t^\ast (\theta, d, \alpha)$ as per se \prettyref{lmm:magnetization} and consider $t \ge t^\ast$.
Define $f = \eexp^{h_u} \prod_{i \in \partial u} (1+ \theta X_{i}^t ) $ and $g = \eexp^{-h_u}\prod_{i \in \partial u} (1 - \theta X_{i}^t )$, and $f'$ and $g'$ are
the corresponding quantities with $X$ replacing by $Y$.  Then it follows from the recursion \prettyref{eq:Magentizationrecursion} and \prettyref{eq:Magentizationrecursionfree} that
\begin{align}
|X_u^{t+1} -Y_{u}^{t+1}  | = \bigg| \frac{ f -  g }{f +  g} -  \frac{f' -  g' }{f' +  g'} \bigg| = 2 \; \bigg| \frac{1}{1+g/f} - \frac{1}{1+g'/f'}\bigg| \le 8 \; \bigg| \left(\frac{g}{f}\right)^{1/4} - \left(\frac{g'}{f'}\right)^{1/4}  \bigg|, \label{eq:magdiff}
\end{align}
where in the last inequality, we apply the following inequality with $s=1/4$:
\begin{align*}
|(1+x)^{-1}-(1+y)^{-1} | \le \frac{1}{s} |x^s - y^s |,
\end{align*}
which holds for all $0<s<1$ and $x,y>0.$
Let $A_i=\left(  \frac{1 - \theta X_{i}^t }{1 + \theta X_{i}^t } \right)^{1/4}$ and $B_i=\left( \frac{1 - \theta Y_{i}^t }{1 + \theta Y_{i}^t } \right)^{1/4}$ for all $i \in \partial u$. Then, it follows from \prettyref{eq:magdiff} that
\begin{align}
\expectp{(X_u^{t+1} -Y_{u}^{t+1} )^2 | \partial u } &\le 64 \expectp{ \eexp^{-h_u} \left( \prod_{i \in \partial u} A_i - \prod_{i  \in \partial u} B_i \right)^2 \bigg| \partial u} \nonumber \\
&= 128 \sqrt{\alpha (1-\alpha) } \expectp{ \left( \prod_{i \in \partial u} A_i - \prod_{i \in \partial u} B_i \right)^2 \bigg| \partial u},
\label{eq:magdiff2}
\end{align}
where the last equality holds because $h_u$ is independent of $\{ A_i, B_i\}_{i \in \partial u}$ conditional on $\tau_u$ and $\partial u$,  and $\expectp{\eexp^{-h_u}}= 2 \sqrt{\alpha (1-\alpha)}.$ 
It is shown in  \cite[Lemma 3.10]{MNS:2013a} than 
\begin{align}
\expectp{ \left( \prod_{i \in \partial u} A_i - \prod_{i \in \partial u} B_i \right)^2  \bigg| \partial u } \le & \frac{1}{2} \binom{D}{2} m^{D-2} \left( \expectp{A^2 - B^2} \right)^2 + D m^{D-1} \expectp{(A-B)^2}, \label{eq:proddiff}
\end{align}
where $D= | \partial u|$, $(A,B)$ has the same distribution as $(A_i,B_i)$ for $i \in \partial u$ and
$m=\max\{\expectp{A^2}, \expectp{B^2} \}.$  We further upper bound the right hand side of \prettyref{eq:proddiff} by
bounding $\expectp{A^2}$, $\expectp{B^2}$, and their differences.
\begin{lemma}\label{lmm:boundA2}
There is a $\theta_1^\ast>0$ such that if $|\theta| \le \theta_1^\ast$ and $\theta^2 d \ge C$ for a sufficiently large universal constant $C$, then
$ m \le 1- \theta^2/4.$
\end{lemma}
\begin{proof}
Note that $(1-x+  5x^2/8)^2 (1+x)= 1-x + x^2/4 + o(x^2)$ when $x \to 0$. 
Thus, there exists a universal constant $\theta_1^\ast>0$ such that
if $|x| \le \theta_1^\ast$, then
$\sqrt{\frac{1-x}{1+x} } \le 1- x+ 5x^2/8$.
It follows that for $\theta \le \theta_1^\ast$ and $i \in \partial u$,
\begin{align*}
\expectp{A_i^2  } & = \expectp{ \sqrt{ \frac{1- \theta X_{i}^t }{ 1+ \theta X_i^t } } } \\
 & \le 1- \theta \expectp{X_{i}^t } + \frac{5\theta^2}{8} \expectp{ (X_i^t)^2 } \\
& \le 1- \theta \expectp{X_{i}^t  } + \frac{5\theta^2}{8},
\end{align*}
where in the last inequality, we used the fact that $|X_{i}^t | \le 1$.
By definition, 
\begin{align*}
\expectp { X_{i}^t } = \expect{X_i^t | \tau_u=+} = (1-\eta) \expect{ X_{i}^t | \tau_{i} = +} + \eta  \expect{ X_{i}^t | \tau_{i} = - }
= (1-2\eta) \expectp{ X_{u}^t },
\end{align*}
where the last equality holds
because the distribution of $-X_{i}^t$ conditional on $\tau_{i}=-$ is the same as  the distribution of $X_{i}^t$ conditional on $\tau_{i}=+$,
and both of them are the same as the distribution of $X_{u}^t$ conditional on $\tau_u=+$.
Hence,
\begin{align*}
\expectp{A_i^2 } \le 1- \theta^2 \expectp{X_{u}^t } + \frac{5\theta^2}{8}.
\end{align*}
In view of \prettyref{lmm:magnetization} and the assumption that $\theta^2 d \ge C$ for a sufficiently large universal constant $C$,
$\expectp{X_{u}^t } \ge 7/8$. Thus  $\expectp{A_i^2 } \le 1- \theta^2/4$. Similarly, we can show
$\expectp{B_i^2 } \le 1- \theta^2/4$ and the lemma follows.
\end{proof}

The following lemma bounds $\expectp{A^2 - B^2 }$ from the above. It is proved
in \cite[Lemma 3.13]{MNS:2013a} and we provide a proof below for completeness.

\begin{lemma}\label{lmm:boundABdiffer}
There is a universal constant $\theta_2^\ast>0$ such that for all $|\theta| \le \theta_2^\ast$,
\begin{align*}
\expectp{A^2 - B^2 } \le 3 \theta^2
\sqrt{\expectp{ (X_u^t - Y_u^t)^2}}.
\end{align*}
\end{lemma}
\begin{proof}
For $i \in \partial u$, the distribution of $A_i$ conditional on $\tau_i=+$ is equal to the distribution of $A_i^{-1}$
conditional on $\tau_i=-$. Hence,
\begin{align}
\expectp{ A_i ^2 } &= (1-\eta) \expect{ A_i^2 | \tau_i=+ } + \eta \expect{ A_i^2 | \tau_i=- } \nonumber \\
& = \expect{ (1-\eta) A_i^2 + \eta A_i ^{-2} | \tau_i=+ }. \label{eq:A2bound1}
\end{align}
Note that
\begin{align}
 (1-\eta) A_i^2 + \eta A_i^{-2} &=  (1-\eta) \left( \frac{1- \theta X_i^t }{ 1+ \theta X_i^t }\right)^{1/2} + \eta  \left( \frac{1+ \theta X_i^t }{ 1- \theta X_i^t }\right)^{1/2}
 \nonumber \\
 & = \frac{ (1-\eta) (1- \theta X_i^t) + \eta (1+ \theta X_i^t) }{ \sqrt{ (1- \theta X_i^t ) (1+ \theta X_i^t  )} }\nonumber \\
 & = \frac{1- \theta^2 X_i^t }{\sqrt{1- \theta^2 (X_i^t)^2 } }, \label{eq:A2bound}
\end{align}
where we used the equality that $1-2\eta= \theta$.  Let $f(x) = \frac{1-\theta^2 x}{ \sqrt{1- \theta^2 x^2 } }$. One can check that there exists a
universal constant $\theta^\ast_2>0$ such that if $|\theta| \le \theta^\ast_2$,
$f'(x) \le 3 \theta^2$ for all $x \in [-1,1]$. Therefore,
\begin{align*}
| f(X_i^t) - f(Y_i^t) | \le 3 \theta^2 |X_i^t - Y_i^t|.
\end{align*}
Combing the last displayed equation with \prettyref{eq:A2bound1} and \prettyref{eq:A2bound} yields that
\begin{align*}
\expectp{ A_i^2 - B_i^2}  \le 3 \theta^2 \expect{|X_i^t - Y_i^t| \; \big| \tau_i=+ } \le 3 \theta^2 \sqrt{ \expect{ (X_i^t - Y_i^t)^2 \; \big| \tau_i=+ } }.
\end{align*}
Finally,  in view of  the fact that $ \expect{(X_i^t - Y_i^t)^2 | \tau_i=+ } =\expectp{ (X_u^t - Y_u^t)^2}$ and
$\expectp{ A_i^2 - B_i^2}  =\expectp{ A^2 - B^2}$, the lemma follows.
\end{proof}

\begin{lemma}\label{lmm:boundABdiffersquare}
There is a universal constant $\theta^\ast_3>0$ such that for all $|\theta| \le \theta^\ast_3$,
\begin{align*}
\expectp{ (A - B)^2} \le \theta^2 \expectp{ (X_u^t - Y_u^t)^2 }.
\end{align*}
\end{lemma}
\begin{proof}
Let $f(x)= (\frac{1-\theta x}{1+\theta x})^{1/4}$ for $x \in [-1,1]$. Then
\begin{align*}
| f'(x) | = \frac{| \theta| }{2 (1+ \theta x)^{5/4} (1-\theta x)^{3/4}} \le \frac{|\theta|}{2 (1- |\theta|)^2}, \quad -1 \le x \le 1.
\end{align*}
Hence, there exists a universal constant $\theta^\ast_3>0$ such that for all $|\theta| \le \theta^\ast_3$, $ |f'(x) | \le |\theta|$ for all $x \in [-1,1]$.
It follows that  for $i \in \partial u$,
$( A_i - B_i )^2 \le \theta^2  \left( X_i^t - Y_i^t \right)^2.$
Therefore,
\begin{align*}
\expectp{ (A - B)^2}  = \expectp{ (A_i - B_i)^2} \le \theta^2 \expectp{  \left( X_i^t - Y_i^t \right)^2}.
\end{align*}
By definition, 
\begin{align*}
\expectp{  \left( X_i^t - Y_i^t \right)^2}&  =(1-\eta) \expect{ \left( X_i^t - Y_i^t \right)^2 | \tau_i=+} + \eta  \expect{ \left( X_i^t - Y_i^t \right)^2 | \tau_i=-} \\
& = \expect{ \left( X_i^t - Y_i^t \right)^2 | \tau_i=+} = \expectp{\left( X_u^t - Y_u^t \right)^2 }.
\end{align*}
The lemma follows by combining the last two displayed equations.
\end{proof}

Finally, we finish the proof of \prettyref{prop:smalltheta}. Let $\theta^\ast = \min \{ \theta^\ast_1, \theta^\ast_2, \theta^\ast_3 \}$.
It follows from \prettyref{lmm:boundA2} that $m \le  1- \frac{\theta^2}{4} \le \eexp^{- \theta^2/4}$.
Assembling \prettyref{eq:proddiff}, \prettyref{lmm:boundABdiffer},  and \prettyref{lmm:boundABdiffersquare} yields that
\begin{align}
\expectp{ \left( \prod_{i \in \partial u} A_i - \prod_{i \in \partial u}  B_i \right)^2 \bigg|   \partial u }
&\le  \left( \frac{9}{4} D^2 \theta^4  \eexp^{-\theta^2 (D-2) /4 } + D \theta^2  \eexp^{-\theta^2 (D-1) /4 } \right) \expectp{ (X_u^t - Y_u^t)^2  } \nonumber \\
& \le c' \eexp^{-\theta^2 D/ 8}  \expectp{ (X_u^t - Y_u^t)^2  } ,
\end{align}
for some universal constant $c'$, where the last inequality holds because $D \theta^2 \eexp^{-\theta^2 D/16}$  is bounded from above by some universal constant.
Since $D \sim \Pois(d)$, it follows that
$\expect{ \eexp^{sD} } = \exp \left( d (\eexp^s -1 ) \right)$.
Thus $\expect{ \eexp^{-\theta^2 D/8} } = \exp \left( d (\eexp^{-\theta^2/8} -1 ) \right)$. Since $\eexp^{ -x } \le 1-x/2$ for $x \in [0,1/2]$, we have that
$\expect{ \eexp^{-\theta^2 D/8} } \le \exp(- d\theta^2 /16 )$. Hence
\begin{align*}
\expectp{ \left( \prod_{i \in \partial u} A_i - \prod_{i \in \partial u}  B_i \right)^2} \le
c'  \eexp^{- d\theta^2 /16 } \; \expectp{ (X_u^t - Y_u^t)^2  }.
\end{align*}
Therefore, there exists a universal constant $C>0$ such that if $d \theta^2 \ge C$, then
\begin{align*}
\expectp{ \left( \prod_{i \in \partial u} A_i - \prod_{i \in \partial u}  B_i \right)^2} \le \frac{1}{128} \expectp{ (X_u^t - Y_u^t)^2  }.
\end{align*}
Combing the last displayed equation together with \prettyref{eq:magdiff2} yields the desired \prettyref{eq:desiredsmalltheta}
and the proposition follows.
\end{proof}

\subsection{Large $|\theta|$ regime}
In this subsection, we focus on the regime where $|\theta| \ge \theta^\ast$. In this regime, to prove \prettyref{thm:main}, it
is sufficient to prove the following proposition.
\begin{proposition}\label{prop:largetheta}
Assume $1/c_0 \le a/b \le c_0$ for some positive constant $c_0$. 
For any $\theta^\ast \in (0,1)$, there exist a $d^\ast =d^\ast(\theta^\ast, c_0)$ such that
for all $\theta \ge \theta^\ast$, $d \ge d^\ast$, and
$t \ge t^\ast(\theta, d, \alpha)$,
\begin{align*}
\expect{ \sqrt{ |X_u^{t+1} -Y_{u}^{t+1} |}  } \le \sqrt{\alpha (1-\alpha)}  \expect{ \sqrt{| X_u^t - Y_u^t | } }.
\end{align*}
\end{proposition}
\begin{proof}
Note that $\expect{ \sqrt{| X_u^t - Y_u^t | }} = \expectp{ \sqrt{| X_u^t - Y_u^t | }}=\expectm{ \sqrt{| X_u^t - Y_u^t | }}$. Thus
it suffices to show that
\begin{align}
\expectp{ \sqrt{ |X_u^{t+1} -Y_{u}^{t+1} |}  } \le \sqrt{\alpha (1-\alpha)}  \expectp{ \sqrt{| X_u^t - Y_u^t | }}. \label{eq:desiredlargetheta}
\end{align}
Fix $t^\ast (\theta, d, \alpha)$ as per se \prettyref{lmm:magnetization} and consider $t \ge t^\ast$.
Let $g_1=\prod_{i \in \partial u} (1+\theta x_i )$, $g_2=\prod_{i \in \partial u} (1- \theta x_i)$, and $g( \{x_i\}_{ i \in \partial u} )= \frac{\eexp^{h_u} g_1 - \eexp^{-h_u} g_2 }{\eexp^{h_u} g_1 + \eexp^{-h_u} g_2}$.
Then $X_{u}^{t+1}=g( \{ X_{i}^t \}_{i \in \partial u} )$ and $Y_{u}^{t+1}=g( \{ Y_{i}^t \}_{i \in \partial u} )$.
Note that $\frac{\partial g_1}{\partial x_i} = \frac{\theta g_1}{1+\theta x_i}$ and $\frac{\partial g_2}{\partial x_i} = -\frac{\theta g_2}{1+\theta x_i}$. It follows that for any $i \in \partial u$,
\begin{align*}
\frac{\partial g}{\partial x_i} = \frac{4 \theta g_1 g_2}{ (1-\theta^2 x_i^2 ) ( g_1 \eexp^{h_u} + g_2 \eexp^{-h_u} )^2}
\le  \frac{4 \theta  g_2}{ (1-\theta^2 x_i^2 ) g_1 \eexp^{2h_u} } =\eexp^{-2h_u} \frac{4\theta}{ (1+\theta x_i)^2} \prod_{j \in \partial u: j \neq i} \frac{1-\theta x_j }{1+\theta x_j},
\end{align*}
where the inequality holds due to $g_2 \ge 0$, and the last equality holds by the definition of  $g_1$ and $g_2$.
By assumption, $\frac{1}{c_0} \le \frac{a}{b} \le c_0$ and thus $|\theta| \le \frac{c_0-1}{c_0+1}$.
Since $(1+\theta x_i ) \ge 1-|\theta| \ge \frac{2}{c_0+1}$  for $|x_i | \le 1$, it follows that for any $i \in \partial u$,
 \begin{align*}
 \frac{\partial g}{\partial x_i} (x) \le  (c_0+1)^2 \eexp^{-2h_u}  \prod_{j \in \partial u: j \neq i} \frac{1-\theta x_j }{1+\theta x_j}  \triangleq  (c_0+1)^2 \eexp^{-2h_u} r_i (x).
 \end{align*}
Note that $r_i (x)$ is convex in $x$ and thus
for any $x, y \in [-1,1]^{|\partial u| }$ and any $0 \le \delta \le 1$,
\begin{align*}
\frac{\partial g}{\partial x_i}  ( \delta x+ (1-\delta) y) \le  (c_0+1)^2 \eexp^{-2h_u} r_i (\delta x+ (1-\delta) y) \le  (c_0+1)^2 \eexp^{-2h_u}\max \{   r_i (x),  r_i (y) \}.
\end{align*}
It follows that
\begin{align*}
| g(x)- g(y) | \le (c_0+1)^2 \eexp^{-2h_u}  \sum_{i \in \partial u} |x_i -y_i |   \max \{  r_i (x),  r_i (y) \}.
\end{align*}
Hence,
\begin{align*}
|X_{u}^{t+1} - Y_u^{t+1} | \le  (c_0+1)^2  \eexp^{-2h_u}  \sum_{i \in \partial u} | X_{i}^t - Y_{i}^t | \max \{  r_i (X), r_i (Y) \}.
\end{align*}
Note that $r_i (X), r_i (Y) $ are functions of  $\{X_{j}^t, Y_{j}^t \}$ for all $j \in \partial u \backslash \{i\}$, and
thus $r_i(X), r_i(Y)$ are independent of $X_i^t$ and $Y_i^t$ conditional on $\tau_u$ and $\partial u$. Moreover,
$h_u$ is independent of $\{X_{i}^t, Y_{i}^t \}_{i \in \partial u}$ conditional on $\tau_u$ and $\partial u$.
Thus it follows from the  last displayed equation that
\begin{align}
& \expectp{ \sqrt{ |X_{u}^{t+1} - Y_u^{t+1} } | \partial u  } \nonumber \\
& \le
 (c_0+1) \expectp{\eexp^{-h_u}} \sum_{i \in \partial u} \expectp{ \sqrt{ | X_{i}^t - Y_{i}^t | } } \expectp{ \max \{  \sqrt{r_i (X)}, \sqrt{r_i (Y)} \} }   \nonumber \\
& = 2  (c_0+1)\sqrt{\alpha (1-\alpha)}  \sum_{i \in \partial u} \expectp{ \sqrt{ | X_{i}^t - Y_{i}^t | } } \expectp{ \max \{  \sqrt{r_i (X)}, \sqrt{r_i (Y)}  \} } \nonumber \\
& = 2  (c_0+1) \sqrt{\alpha (1-\alpha)} \; D \; \expectp{ \sqrt{ | X_{u}^t - Y_{u}^t | } }  \expectp{ \max \{  \sqrt{r (X)}, \sqrt{r (Y)}  \} }, \label{eq:magdiff3}
\end{align}
where $D= |\partial u|$; $( r(X), r(Y)) $ has the same distribution as $(r_i(X), r_i(Y))$ for $i \in \partial u$;
the first equality follows due to $\expectp{\eexp^{-h_u}} = 2\sqrt{\alpha (1-\alpha)} $;
the last equality holds because for all $i \in \partial u$,
\begin{align*}
\expectp{ \sqrt{ | X_{i}^t - Y_{i}^t | } }  & = (1-\eta) \expect{  \sqrt{ | X_{i}^t - Y_{i}^t | } | \tau_i=+ }+ \eta \expect{  \sqrt{ | X_{i}^t - Y_{i}^t | } | \tau_i=- }  \\
& =  \expect{  \sqrt{ | X_{i}^t - Y_{i}^t | } | \tau_i=+ } =\expectp{ \sqrt{ | X_{u}^t - Y_{u}^t | } }.
\end{align*}
To proceed, we need to bound  $\expectp{ \max \{  \sqrt{r (X)}, \sqrt{r (Y)}  \} }$ from the above. 
\begin{lemma}\label{lmm:contraction}
 Assume $1/c_0 \le a/b \le c_0$ for some positive constant $c_0$. For any $0< \theta^\ast <1$, there exists a $d^\ast = d^\ast (\theta^\ast, c_0)$ and a $\lambda= \lambda(\theta^\ast)  \in (0,1)$ such that
for all $|\theta| \ge \theta^\ast$, $d \ge d^\ast$, $t \ge t^\ast(\theta, d, \alpha)$, and $i \in \partial u$,
\begin{align*}
\expectp{ \sqrt{ \frac{1-\theta X_{i}^t }{1+ \theta X_{i}^t } } } \le \lambda  ,
\end{align*}
and the same is also true for $Y_i^t$.
\end{lemma}
\begin{proof}
By assumption $\frac{1}{c_0} \le \frac{a}{b} \le c_0$,  it follows that $|\theta| \le \frac{c_0-1}{c_0+1}$ and $ \frac{1}{c_0+1} \le \eta \le \frac{c_0}{c_0+1}$.
We prove the claim for $X_i^t$; the claim for $Y_i^t$ follows similarly.
Fix $\epsilon=\epsilon( \theta^\ast, c_0)<1$ to be determined later. It follows from \prettyref{lmm:magnetization} and the Markov inequality that
for all $t \ge  t^\ast(\theta, d, \alpha)$
\begin{align*}
\prob{1-X_u^t \ge \epsilon | \sigma_u=+ } \le \frac{ 1- \expectp{X_u^t} }{ \epsilon } \le \frac{2 C \eta}{\theta^2 d \epsilon} + \frac{2 \eexp^{-C d} }{\epsilon},
\end{align*}
where $C$ is a universal constant.  Hence, there exists a $d^\ast = d^\ast (\theta^\ast, c_0)$ such that for all $d \ge d^\ast$ and $\theta \ge \theta^\ast$,
$\prob{1-X_u^t \ge \epsilon | \sigma_u=+ }  \le \epsilon$, \ie, $\prob{X_u^t \ge 1- \epsilon | \sigma_u=+ } \ge 1-\epsilon$.
Since the distribution of $X_u^t$ conditional
on $\sigma_u$ is the same as the distribution of $X_i^t$ conditional on $\sigma_i$ for $i \in \partial u$, it follows that $\prob{X_i^t \ge 1- \epsilon | \sigma_i=+ } \ge 1-\epsilon$.
By symmetry,  the distribution of $X_i^t$ conditional on $\sigma_i=-$ is the same as the distribution of $-X_i^t$ conditional on $\sigma_i=+$, 
and thus $\prob{X_i^t \le \epsilon -1 | \sigma_i=- } \ge 1-\epsilon$.
Hence,
\begin{align*}
\prob{X_i^t \ge 1- \epsilon | \sigma_u=+ } & \ge \prob{X_i^t \ge 1- \epsilon, \sigma_i=+ | \sigma_u=+ }   \\
& = \prob{ \sigma_i=+ | \sigma_u=+} \prob{X_i^t \ge 1- \epsilon | \sigma_i=+}  \\
& \ge (1-\eta) (1-\epsilon) \ge 1-\eta -\epsilon.
\end{align*}
Similarly,
\begin{align*}
\prob{X_i^t \le \epsilon -1 | \sigma_u=+ } & \ge  \prob{X_i^t \le  \epsilon -1, \sigma_i=- | \sigma_u=+ }  \\
& = \prob{ \sigma_i= - | \sigma_u=+} \prob{X_i^t \le  \epsilon -1 | \sigma_i=-} \\
& \ge \eta (1-\epsilon) \ge \eta -\epsilon.
\end{align*}
Let $f(x)= \sqrt{ \frac{1- \theta x }{1+ \theta x} }$. Then $f$ is non-increasing in $x$ if $\theta \ge 0$ and non-decreasing if $\theta<0$. It follows that 
\begin{align*}
\expectp{ f(X_i^t) } & \le  f ( 1-\epsilon ) (1- \eta -\epsilon) + f(-1) (\eta +\epsilon),  \quad \theta \ge 0, \\
\expectp{ f(X_i^t) }  & \le  f ( 1 ) (1- \eta + \epsilon) + f(\epsilon -1 ) (\eta - \epsilon), \quad \theta<0.
\end{align*}
Notice that
\begin{align*}
f(1-\epsilon) \le f(1) + \epsilon \sup_{x \in [1-\epsilon, 1]} |f'(x) |  & \le \sqrt{\frac{ \eta}{1-\eta}} + \frac{(c_0+1)^2 \epsilon }{4}, \\
f( \epsilon -1) \le f(-1) + \epsilon \sup_{x \in [-1, \epsilon-1] } | f'(x) | & \le \sqrt{\frac{ 1- \eta}{\eta}} + \frac{(c_0+1)^2 \epsilon }{4},
\end{align*}
where the last inequality holds because  for $x \in [-1, 1]$,
\begin{align*}
| f'(x) | = \frac{\theta^2 |x| }{ (1+ \theta x)^{3/2} (1-\theta x)^{1/2} } \le \frac{1 } { (1-|\theta| )^{2} } \le  \frac{(c_0+1)^2}{4}.
\end{align*}
Hence,
\begin{align*}
f ( 1-\epsilon ) (1- \eta -\epsilon) + f(-1) (\eta +\epsilon) & \le
f (1) (1-\eta) + f(-1) \eta +  \frac{(c_0+1)^2 \epsilon }{4} + \epsilon f(-1) \\
& = 2 \sqrt{ \eta (1-\eta) }  +  \frac{(c_0+1)^2 \epsilon }{4} + \epsilon f(-1) \\
& \le \sqrt{1 - (\theta^\ast)^2 } + \frac{(c_0+1)^2 \epsilon }{4} + \epsilon \sqrt{c_0},
\end{align*}
where the last inequality holds because  $\eta = \frac{1-\theta}{2}$, $|\theta| \ge \theta^\ast$, and $f(-1) \le \sqrt{c_0}$.
Similarly,
\begin{align*}
 f ( 1 ) (1- \eta + \epsilon) + f(\epsilon -1 ) (\eta - \epsilon) & \le
f (1) (1-\eta + \epsilon ) + f(-1) \eta +  \frac{(c_0+1)^2 \epsilon  }{4}  \\
& = 2 \sqrt{ \eta (1-\eta) }  +  \frac{(c_0+1)^2 \epsilon }{4} + \epsilon f(1) \\
& \le \sqrt{1 - (\theta^\ast)^2 } + \frac{(c_0+1)^2 \epsilon }{4} + \epsilon \sqrt{c_0}.
\end{align*}
In conclusion, for both of case $\theta \ge 0$ and case $\theta<0$, we have shown that
\begin{align*}
\expectp{ f(X_i^t) } \le \sqrt{1 - (\theta^\ast)^2 } + \frac{(c_0+1)^2 \epsilon }{4} + \epsilon \sqrt{c_0}.
 \end{align*}
Therefore,
there exists an $\epsilon^\ast=\epsilon^\ast( \theta^\ast, c_0)$ such that $\expectp{ f(X_i^t) } \le \lambda$ for some $\lambda=\lambda(\theta^\ast) \in (0,1) $.
\end{proof}
Finally, we finish the proof of \prettyref{prop:largetheta}. It follows from \prettyref{lmm:contraction} that
\begin{align*}
\expectp{ \max \{  \sqrt{ r(X)}, \sqrt{r (Y)} \}}  \le \expectp{  \sqrt{r (X)} + \sqrt{r(Y)}  }  \le 2 \lambda^{D-1}.
\end{align*}
Combing the last displayed equation with \prettyref{eq:magdiff3} yields
\begin{align*}
 \expectp{ \sqrt{ |X_{u}^{t+1} - Y_u^{t+1} } | \partial u  }   \le 4 (c_0+1)   \sqrt{\alpha (1-\alpha)}  D  \lambda^{D-1}  \expectp{ \sqrt{ | X_{u}^t - Y_{u}^t | } }  .
\end{align*}
Thus,
\begin{align*}
 \expectp{ \sqrt{ |X_{u}^{t+1} - Y_u^{t+1} } }   \le  4 (c_0+1)  \sqrt{\alpha (1-\alpha)}  \expect{D   \lambda^{D-1}}  \expectp{ \sqrt{ | X_{u}^t - Y_{u}^t | } }.
\end{align*}
By Cauchy-Schwarz inequality and $D \sim \Pois(d)$,
\begin{align*}
\expect{D   \lambda^{D} } \le \sqrt{\expect{D^2} } \sqrt{\expect{\lambda^{2D}} } = \sqrt{2d} \eexp^{d(\lambda^2 -1)/2}.
\end{align*}
Combing the last two displayed equation, we get that
\begin{align*}
 \expectp{ \sqrt{ |X_{u}^{t+1} - Y_u^{t+1} } }   \le  4 (c_0+1) \sqrt{d} \eexp^{d(\lambda^2 -1)/2} \lambda^{-1} \sqrt{\alpha (1-\alpha)}  \expectp{ \sqrt{ | X_{u}^t - Y_{u}^t | } }.
\end{align*}
Since $\lambda=\lambda(\theta^\ast) \in (0,1)$, there exists a $d^\ast(\theta^\ast, c_0)$ such that for all $d \ge d^\ast$, 
the desired \prettyref{eq:desiredlargetheta} holds
and hence the proposition follows.
\end{proof}

\section{Optimality of local BP  when $\alpha$ is small} \label{sec:accurateSI}
\begin{theorem} \label{thm:regularaccurateSI}
There exists a constant $0<\alpha^\ast<1/2 $ depending only on $c_0$ such that if $\alpha \le \alpha^\ast$, then
\begin{align*}
\lim_{t \to \infty} \limsup_{n \to \infty} \expect{|X_u^t- Y_u^t| } =0.
\end{align*}
\end{theorem}
\begin{proof}
The proof is similar to the proof of \prettyref{thm:main} and is divided into three steps. Thus we only provide proof sketches below.

We first show that for all $t \ge 0$,
\begin{align}
\min\left\{ \expectp{X_{u}^t } , \expectp{Y_{u}^t }  \right\} \ge 1 - 4 \alpha. \label{eq:magboundalpha}
\end{align}
In particular, given access to $\tau_{\partial T_u^t}$ and $\tilde{\tau}_{T_u^t}$, the optimal estimator of $\tau_u$ is $\sgn(X_u^t)$ whose success probability
is  $(1+\expect{|X_u^t |} )/2$. For the estimator $\tilde{\tau}_u$, its success probability is $1-\alpha$. It follows that
$\expect{|X_u^t |} \ge  1-2 \alpha.$
In view of  \prettyref{lmm:expectedmagnetization} , $\mathbb{E}^{+} [ X_u^t ] \ge 2 \expect{ |X_u^t |}  - 1 \ge 1- 4 \alpha$.
By the same argument,  $\mathbb{E}^{+} [ Y_u^t ]  \ge 1- 4 \alpha$.

We next focus on the case $|\theta| \le \theta^\ast$, where $\theta^\ast$
is a small constant to be specified later. To prove the theorem in this case, it is sufficient to  show that
\begin{align}
\expectp{(X_u^{t+1} -Y_{u}^{t+1} )^2} \le C \sqrt{\alpha (1-\alpha)}  \expectp{(X_u^t - Y_u^t )^2} \label{eq:smallalphasmalltheta}
\end{align}
for some explicit universal constant $C$.
In the proof of \prettyref{prop:smalltheta}, we have shown that
\begin{align*}
\expectp{(X_u^{t+1} -Y_{u}^{t+1} )^2  } \le
128 \sqrt{\alpha (1-\alpha) } \expectp{ \left( \prod_{i=1}^d A_i - \prod_{i=1}^d B_i \right)^2},
\end{align*}
and
\begin{align*}
\expectp{ \left( \prod_{i=1}^d A_i - \prod_{i=1}^d B_i \right)^2 \bigg| \partial u } \le  \frac{1}{2} \binom{D}{2} m^{D-2} \left( \expectp{A^2 - B^2 } \right)^2 + D m^{D-1} \expectp{(A-B)^2}.
\end{align*}
Following the proof of \prettyref{lmm:boundA2}, one can check that there exists a $\theta_1^\ast$ such that if $|\theta| \le \theta_1^\ast$
and $\alpha \le 1/32$, then $m \le  1- \frac{\theta^2}{4} \le \eexp^{- \theta^2/4}$.
In view of \prettyref{lmm:boundABdiffer} and \prettyref{lmm:boundABdiffersquare}, there exist $\theta_2^\ast$ and $\theta_3^\ast$ such that
if $|\theta| \le \theta^\ast \triangleq \min\{\theta_1^\ast, \theta_2^\ast, \theta_3^\ast \}$, then
\begin{align*}
\expectp{ \left( \prod_{i=1}^d A_i - \prod_{i=1}^d B_i \right)^2 \bigg| \partial u }
& \le  \left( \frac{9}{4} D^2 \theta^4  \eexp^{-\theta^2 (D-2) /4 } + d \theta^2  \eexp^{-\theta^2 (D-1) /4 } \right) \expectp{ (X_u^t - Y_u^t)^2  } \\
& \le C' \; \expectp{ (X_u^t - Y_u^t)^2  },
\end{align*}
where $C'$ is a universal constant and the last inequality holds because
$ D \theta^2  \eexp^{-\theta^2 D}$ and $D^2 \theta^4  \eexp^{-\theta^2 D }$ are bounded by a universal constant from above. Combing the last three displayed equations 
yields the desired \prettyref{eq:smallalphasmalltheta}.

Finally, we consider  the case $|\theta| \ge \theta^\ast$. In this case, to prove the theorem, it suffices to show that
\begin{align}
\expectp{ \sqrt{ |X_u^{t+1} -Y_{u }^{t+1} |}} \le C'(c_0) \sqrt{\alpha (1-\alpha)}  \expectp{ \sqrt{| X_u^t - Y_u^t | } }, \label{eq:contractionlargethetaalpha}
\end{align}
for some explicit constant $C'$ depending only on $c_0$.
In the proof of \prettyref{prop:largetheta}, we have shown that \prettyref{eq:magdiff3} holds, which gives 
\begin{align*}
\expectp{ \sqrt{ |X_{u}^{t+1} - Y_u^{t+1} } | \partial u  }  \le
2  (c_0+1) \sqrt{\alpha (1-\alpha)}  D \; \expectp{ \sqrt{ | X_{u}^t - Y_{u}^t | } }  \expectp{ \max \{  \sqrt{r (X)}, \sqrt{r (Y)}  \} },
\end{align*}
where $( r(X), r(Y)) $ has the same distribution as $(r_i(X), r_i(Y))$ for $i \in \partial u$, 
and $ r_i (x)=\prod_{j \in \partial u: j \neq i} \frac{1-\theta x_j }{1+\theta x_j}.$
Following the proof of \prettyref{lmm:contraction}, one can verify that  there exists an $\alpha^\ast(\theta^\ast, c_0)$ and a $\lambda(\theta^\ast ) \in (0,1)$ such that
if $\alpha \le \alpha^\ast$, then  for all $i \in \partial u$,
\begin{align*}
\max\left\{ \expectp{ \sqrt{ \frac{1-\theta X_{i}^t }{1+ \theta X_{i}^t } }  },  \expectp{ \sqrt{ \frac{1-\theta Y_{i}^t }{1+ \theta Y_{i}^t } }  } \right\} \le \lambda.
\end{align*}
It follows that
\begin{align*}
\expectp{ \max \{  \sqrt{ r(X)}, \sqrt{r (Y)} \}}  \le \expectp{  \sqrt{r (X)} + \sqrt{r(Y)}  }  \le 2 \lambda^{D-1}.
\end{align*}
Combing the last three displayed equations yields
\begin{align*}
 \expectp{ \sqrt{ |X_{u}^{t+1} - Y_u^{t+1} } | \partial u  }   &\le 4 (c_0+1) \sqrt{\alpha (1-\alpha)}  D  \lambda^{D-1}   \expectp{ \sqrt{ | X_{u}^t - Y_{u}^t | } } ,
\end{align*}
which implies the desired \prettyref{eq:contractionlargethetaalpha} holds, because $D \lambda^{D-1}$ is bounded  by a universal constant from above. 
\end{proof}

\section{Density evolution in the large degree regime} \label{sec:density evolution}
 In this section, we consider the regime \prettyref{eq:asymptotics} and further assume that as $n \to \infty$,
 \begin{align*}
 a \to \infty \qquad \frac{a-b}{\sqrt{b} } \to \mu,
 \end{align*}
 where $\mu$ is a fixed constant.
For $t \ge 1$, define
 \begin{align}
 \Phi_u^t &
 = \sum_{\ell \in \partial u } F( \Phi_\ell^{t-1} + h_\ell) \label{eq:recursionPhi},  \\
 \Psi_u^t &
 = \sum_{\ell \in \partial u } F( \Psi_\ell^{t-1} + h_\ell ) \label{eq:recursionPsi} ,
 \end{align}
 where $ \Phi_u^{0} =\infty$ if $\tau_u=+$ and $ \Phi_u^{0} =-\infty$ if $\tau_u=-$; $\Psi_u^0=0$ for all $u$.
Then $\Lambda_u^t=\Phi_u^t + h_u$ and  $\Gamma_u^t= \Psi_u^t + h_u$ for all $t \ge 0$.
Notice that subtrees $\{T_\ell^t\}_{\ell \in \partial u}$ are independent and identically distributed conditional on $\tau_u$.
Thus $\{ \Phi_{\ell  }^{t-1} \}_{\ell \in \partial u} $ ($\{ \Psi_{\ell  }^{t-1} \}_{\ell \in \partial u} $) are independent and identically distributed conditional on $\tau_u$.
As a consequence, when the expected degree of $u$ tends to infinity, due to the central limit theorem,
 we expect that the distribution of $\Phi_u^t$ ($\Psi_u^t$) conditional on $\tau_u$ is approximately Gaussian.

 Let $W_+^t$ ($Z_+^t$) denote a random variable that has the same distribution as $\Phi_u^t$ ($\Psi_u^t$)  conditional on $\tau_u=+$,
and $W_-^t$ ($Z_-^t$) denote a random variable that has the same distribution as $\Phi_u^t$ ($\Psi_u^t$)  conditional on $\tau_u=-$.
We are going to  prove that the distributions of $W_+^t$ ($Z_+^t$)  and $W_-^t$ ($Z_-^t$)   are asymptotically Gaussian.
The following lemma provides expressions of the mean and variance of $Z_+^t$ and $Z_-^t$.
\begin{lemma}
For all $t \ge 0$,
\begin{align}
\expect{Z_{\pm}^{t+1}} & =  \pm\frac{ \mu^2 }{4}  \expect{ \tanh(Z_+^t + U  )}+ O( a^{-1/2} ). \label{eq:meanZ}\\
\var\left( Z_{\pm}^{t+1}\right) & =  \frac{ \mu^2 }{4} \expect{ \tanh(Z_+^t + U  ) } + O( a^{-1/2} ).\label{eq:varianceZ}
\end{align}
\label{lmm:meanvarianceZ}
\end{lemma}
\begin{proof}
By symmetry, the distribution of $\Gamma_u^t$ conditional on $\tau_u=-$ is the same as the distribution of $-\Gamma_u^t$ conditional on $\tau_u=+$.
Thus, $\expect{Z_{+}^{t+1}} = - \expect{Z_{-}^{t+1}}$ and $\var\left( Z_{+}^{t+1}\right) =\var\left( Z_{-}^{t+1}\right)$.
Hence, it suffices to prove the claims for $Z_-^{t+1}$.
By the definition of $\Gamma_u^t$ and the change of measure, we have that
 \begin{align*}
 \expect{ g \left( \Gamma_u^t \right) | \tau_u= - 1 } = \expect{ g \left(  \Gamma_u^t \right) \eexp^{ - 2 \Gamma_u^t } | \tau_u=+1},
 \end{align*}
 where $g$ is any measurable function such that the expectations above are well-defined. Recall that
 $\Gamma_u^t =h_u + \Psi_u^t$.
 Hence, the distribution of $\Gamma_u^t$ conditional on $\tau_u=+$
 is the same as the distribution of $ U + Z_{+}^t$; the distribution of $\Gamma_u^t$ conditional on $\tau_u=-$
 is the same as the distribution of $ -U + Z_{-}^t$.
  It follows that
 \begin{align}
 \expect{ g \left( Z_-^t -U   \right) } = \expect{ g \left(  Z_+^t +U  \right) \eexp^{-2 ( Z_+^t +U ) } }. \label{eq:nishimori}
 \end{align}
Define $\psi(x) \triangleq  \log (1+x)- x + x^2/2$.
It follows from the Taylor expansion that
$ |\psi(x) | \le |x|^3$. Then
\begin{align*}
 F(x) & = \frac{1}{2} \log \left(  \frac{ \eexp^{2x+2\beta}    +1   }{ \eexp^{2x}  + \eexp^{2\beta} }  \right)
  = -\beta + \frac{1}{2} \log \left( \frac{ \eexp^{2x+2\beta}    +1   }{ \eexp^{2x -2 \beta }  + 1 } \right) \\
 & = -\beta  + \frac{1}{2} \log \left( 1 + \frac{ \eexp^{4\beta}  -1 }{1+ \eexp^{-2 (x -\beta) }  } \right)\\
 & = -\beta +  \frac{\eexp^{4\beta}  -1 }{2} f(x) - \frac{ \left( \eexp^{4\beta}  -1 \right)^2 }{4} f^2(x) + \frac{1}{2} \psi\left( (\eexp^{4\beta}  -1) f(x)\right),
\end{align*}
where $f(x)= \frac{1}{1+ \eexp^{-2(x -\beta) } }.$ Since $ |\psi(x) | \le |x|^3$ and $|f(x)| \le 1$, it follows that
\begin{align}
 F(x)
  = -\beta +  \frac{\eexp^{4\beta}  -1 }{2} f(x) - \frac{ \left( \eexp^{4\beta}  -1 \right)^2 }{4} f^2(x) + O \left( |\eexp^{4\beta}  -1|^3 \right), \label{eq:ApproxF}
\end{align}
Therefore,
\begin{align*}
\Psi_u^{t+1} & = \sum_{\ell \in \partial u } F( \Psi_\ell^{t } + h_\ell ) \\
& = \sum_{\ell \in \partial u }  \left[ -\beta+ \frac{\eexp^{4\beta}  -1 }{2} f(  \Psi_\ell^{t } + h_\ell )
-   \frac{ \left( \eexp^{4\beta}  -1 \right)^2 }{4} f^2( \Psi_\ell^{t } + h_\ell ) +  O \left( |\eexp^{4\beta}  -1|^3 \right)  \right].
\end{align*}
By conditioning the label of vertex $u$ is $-$, it follows that
\begin{align*}
\expect{Z_-^{t+1}} &=  - \beta \frac{a+b}{2}  +  \frac{\eexp^{4\beta}  -1 }{4} \left( b  \expect{  f (Z_+^t + U )    } +  a  \expect{ f (Z_-^t - U ) }
\right) \\
& - \frac{\left( \eexp^{4\beta}  -1 \right)^2  }{8} \left(  b \expect{ f^2 (Z_+^t + U ) } +  a \expect{ f^2 (Z_-^t - U ))} \right) +  O \left(b  |\eexp^{4\beta}  -1|^3 \right).
\end{align*}
In view of \prettyref{eq:nishimori}, we have that
\begin{align}
 b  \expect{  f (Z_+^t + U )    } +  a  \expect{ f (Z_-^t - U ) }  & =   b \expect{f (Z_+^t + U ) ( 1+ \eexp^{-2( Z_+^t +U  - \beta) } )} =b,   \label{eq:nishimori1}  \\
 b \expect{ f^2 (Z_+^t + U ) } +  a \expect{ f^2 (Z_-^t - U ) }  & = b \expect{ f^2 (Z_+^t + U ) ( 1+ \eexp^{-2( Z_+^t+U  - \beta ) } )} = b \expect{f (Z_+^t + U )}.\label{eq:nishimori2}
\end{align}
Hence,
\begin{align*}
\expect{Z_-^{t+1}} =  - \beta \frac{a+b}{2}  +   \frac{ b \left( \eexp^{4\beta}  -1 \right) }{4}
- \frac{b \left( \eexp^{4\beta}  -1 \right)^2  } {8}  \expect{f (Z_+^t + U )}  +  O \left( b |\eexp^{4\beta}  -1|^3 \right) .
\end{align*}
Notice that
\begin{align}
-\beta = - \frac{1}{2} \log \left( 1 + \frac{a-b}{b} \right) = - \frac{a-b}{2b} + \frac{(a-b)^2}{4 b^2 } + O\left(  \frac{|a-b|^3}{ b^3 }\right). \label{eq:ApproxBeta}
\end{align}
As a consequence,
\begin{align*}
- \beta \frac{a+b}{2}  +   \frac{ b \left( \eexp^{4\beta}  -1 \right)  }{4}
&= -  \frac{a^2-b^2}{4b} + \frac{(a-b)^2 (a+b) }{ 8 b^2} + \frac{a^2-b^2}{4 b} + O\left( \frac{|a-b|^3}{ b^2 }\right) \\
&=  \frac{(a-b)^2 }{4b} + O\left( \frac{|a-b|^3}{ b^2 }\right)= \frac{\mu^2}{4} + O ( a^{-1/2}),
\end{align*}
where the last equality holds due to $(a-b)/\sqrt{b} \to \mu$ for a fixed constant $\mu$.
Moreover,
\begin{align*}
\frac{b \left( \eexp^{4\beta}  -1 \right)^2  } {8}  = \frac{(a^2-b^2)^2 }{8b^3} = \frac{(a-b)^2}{2b} \left( 1+ \frac{a-b}{2b} \right)^2 =\mu^2/2 + O (a^{-1/2} ),
\end{align*}
and
\begin{align*}
b |\eexp^{4\beta}  -1|^3 = O\left( \frac{|a-b|^3}{ b^2 }\right) =  O ( a^{-1/2}).
\end{align*}
Assembling the last five displayed equations gives that
\begin{align*}
\expect{Z_-^{t+1}} =  \frac{\mu^2 }{4}  - \frac{ \mu^2 }{2}  \expect{f (Z_+^t + U )} + O( a^{-1/2} ).
\end{align*}
Finally, notice that
\begin{align}
\bigg| f(x) -\frac{1}{1+ \eexp^{-2x } }  \bigg| =  \frac{\eexp^{-2x} \big| \eexp^{2\beta} -1\big| }{ ( 1+ \eexp^{-2(x -\beta) } ) ( 1+ \eexp^{-2x } ) }\le  \big|\eexp^{2\beta} -1 \big| = O( a^{-1/2} ). \label{eq:Approxf}
\end{align}
It follows that
\begin{align*}
\expect{Z_-^{t+1}} &=  \frac{\mu^2 }{4}  - \frac{ \mu^2 }{2}  \expect{ \frac{1}{1+ \eexp^{-2 (Z_+^t + U ) } }} + O( a^{-1/2} ) \\
& =  - \frac{ \mu^2 }{4} \tanh(Z_+^t + U  )+ O( a^{-1/2} ).
\end{align*}

Next we calculate $\var(Z_-^{t+1})$. For $Y= \sum_{i=1}^L X_i$, where $L$ is Poisson distributed, and $\{ X_i \}$ are i.i.d.\ with finite second moments,
one can check that $\var(Y) = \expect{L} \expect{X_1^2}$.
Since $\Psi_u^{t+1}  = \sum_{\ell \in \partial u } F( \Psi_\ell^{t } + h_\ell )$, it follows that
\begin{align*}
\var (Z_-^{t+1}) = \frac{b}{2} \expect{ F^2 (Z_+^t +U ) } + \frac{a}{2} \expect{ F^2 ( Z_-^t -U ) },
\end{align*}
In view of \prettyref{eq:ApproxF} and the fact that $\eexp^{4\beta} -1 =o(1)$, we have that
\begin{align*}
F^2(x) = \beta^2 - \left(\eexp^{4\beta} -1 \right) \beta f(x) +  \frac{ \left( \eexp^{4\beta}  -1 \right)^2 (2\beta+1) }{4} f^2(x)  + O \left( |\eexp^{4\beta}  -1|^3 \right),
\end{align*}
Thus,
\begin{align*}
\var (Z_-^{t+1}) & =  \beta^2 \frac{a+b}{2} - \frac{\left(\eexp^{4\beta} -1 \right) \beta }{2} \left[ b \expect{ f (Z_+^t +U ) } + a\expect{ f (Z_-^t -U ) }   \right] \\
& + \frac{ \left( \eexp^{4\beta}  -1 \right)^2 (2\beta+1) }{8}  \left[ b \expect{ f^2 (Z_+^t +U ) } + a\expect{ f^2 (Z_-^t -U ) }   \right] + O \left(b |\eexp^{4\beta}  -1|^3 \right).
\end{align*}
Applying \prettyref{eq:nishimori1} and \prettyref{eq:nishimori2}, we get that
\begin{align*}
\var (Z_-^{t+1}) =  \beta^2 \frac{a+b}{2} - \frac{\left(\eexp^{4\beta} -1 \right) \beta b }{2}
 + \frac{ \left( \eexp^{4\beta}  -1 \right)^2 (2\beta+1) b }{8}  \expect{ f (Z_+^t +U ) } + O \left(b |\eexp^{4\beta}  -1|^3 \right).
\end{align*}
In view of \prettyref{eq:ApproxBeta}, we have that
\begin{align*}
 \beta^2 \frac{a+b}{2} - \frac{\left(\eexp^{4\beta} -1 \right) \beta b }{2} & = \frac{ (a-b)^2 (a+b)  }{8b^2} -  \frac{ (a-b)^2(a+b)}{4b^2} + O \left(  \frac{|a-b|^3}{b^2}\right)\\
&=  - \frac{ (a-b)^2   }{4b} +  O \left(  \frac{|a-b|^3}{b^2}\right) = - \frac{\mu^2}{4} + O(a^{-1/2} ) ,
\end{align*}
and that
\begin{align*}
\frac{ \left( \eexp^{4\beta}  -1 \right)^2 (2\beta+1) b }{8}  &=  \frac{(a-b)^2}{2b} \left(1 + \frac{a-b}{2b} \right)^2 (2\beta+1)  \\
 & =  \frac{(a-b)^2}{2b} + O \left(  \frac{|a-b|^3}{b^2}\right) =\frac{\mu^2}{2} +O(a^{-1/2} ).
\end{align*}
Moreover, we have shown that $b |\eexp^{4\beta}  -1|^3=O(a^{-1/2})$.
Assembling the last three displayed equations gives that
\begin{align*}
\var (Z_-^{t+1}) = - \frac{\mu^2}{4} + \frac{\mu^2}{2} \expect{ f (Z_+^t +U ) } + O(a^{-1/2}).
\end{align*}
Finally, in view of \prettyref{eq:Approxf}, we get that
\begin{align*}
\var (Z_-^{t+1}) &=  - \frac{\mu^2}{4} + \frac{\mu^2}{2} \expect{  \frac{1}{1+ \eexp^{- 2( Z_+^t +U) }  } }  +  O(a^{-1/2} ) \\
&= \frac{ \mu^2  }{4}  \expect{  \tanh( Z_+^t+ U ) } + O(a^{-1/2} ).
 \end{align*}
\end{proof}

The following lemma is useful for proving the distributions of $Z_+^t$ and $Z_-^t$ are approximately Gaussian.
\begin{lemma}(Analog of Berry-Esseen inequality for Poisson sums \cite[Theorem 3]{korolev2012improvement}.)\label{lmm:Poisson_BE}
Let  $S_{\nu}=X_1 + \cdots + X_{N_\nu},$   where
$X_i: i\geq 1$  are independent, identically distributed random variables with finite second moment,
and $\expect{|X_i|^3}\leq \rho^3,$ and for some $\nu > 0,$ $N_{\nu}$ is a $\Pois(\nu)$ random variable independent
of  $(X_i: i\geq 1).$   Then
$$
\sup_x \bigg|     \prob{  \frac{S_\nu - \nu \expect{X_1} }{  \sqrt{\nu \expect{X_1^2}  }}\leq x} - \prob{Z \leq x}  \bigg|  \leq  \frac{C_{BE} \rho^3}{\sqrt{\nu ( \expect{X_1^2} )^3}} , 
$$
where $C_{BE}=0.3041.$
\end{lemma}


 \begin{lemma}\label{lmm:gaussiandensityevolution}
Suppose $\alpha \in (0,1/2]$ is fixed. Let $h(v)= \expect{ \tanh ( v + \sqrt{v} Z + U )}$,
where $Z \sim \calN(0,1)$ and $U=\gamma$ with probability $1-\alpha$ and $U=-\gamma$
with probability $\alpha$, where $\gamma= \frac{1}{2} \log \frac{1-\alpha}{\alpha}$. 
Define $(v_t: t \ge 0)$ recursively by $v_0=0$ and
$
v_{t+1} = \frac{\mu^2}{4} h(v_t).
$
For any fixed $t \ge 0$, as $n \to \infty$,
\begin{align}
\sup_x \bigg|     \prob{  \frac{Z_{\pm}^{t}  \mp v_{t}  }{  \sqrt{ v_t } } \leq x } - \prob{Z \leq x}    \bigg|  =
O(a^{-1/2}).
 \label{eq:gaussianlimit}
\end{align}
Define $(w_t: t \ge 1)$ recursively by $w_1=\mu^2/4$ and $
w_{t+1} = \frac{\mu^2}{4} h(w_t)$.
For any fixed $t \ge 1$, as $n \to \infty$,
\begin{align}
\sup_x \bigg|     \prob{  \frac{W_{\pm}^{t}  \mp w_{t}  }{  \sqrt{ w_t } } \leq x } - \prob{Z \leq x}    \bigg|  =
O(a^{-1/2}). \label{eq:gaussianlimitw}
\end{align}
 \end{lemma}
 \begin{proof}
 We prove the lemma by induction over $t$.  We first consider the base case. For $Z^t$, the base case $t=0$
trivially holds, because  $\Psi_u^0 \equiv 0$ and $v_0=0$. For $W^t$, we need to check the base case $t=1$.
Recall that $\Lambda^0_\ell=\infty$ if $\tau_\ell=+$ and  $\Lambda^0_\ell= -\infty$ if $\tau_\ell=-$.
Notice that $F(\infty)=\beta$ and $F(-\infty)=-\beta$. Hence,
$
\Phi_u^1 = \sum_{i=1}^{N_{d} } X_i,
$
where $N_d \sim \Pois(d)$ is independent of $\{X_i\}$; $\{X_i\}$ are i.i.d.\ such that
conditional on $\tau_u$, $ X_i=\tau_u\beta$ with probability $a/(a+b)$ and $X_i=-\tau_u\beta$ with probability $b/(a+b)$.
Consequently, $\expect{X_1|\tau_u}=\beta \theta \tau_u$, $\expect{X_1^2} = \beta^2$ and $\expect{|X_1|^3} =|\beta|^3$.
Thus, in view of \prettyref{lmm:Poisson_BE}, we get that
\begin{align}
\sup_x \bigg|     \prob{  \frac{ W_{\pm}^1 \mp d \beta \theta }{  \sqrt{d  \beta^2  }}\leq x} -  \prob{Z \leq x}   \bigg|  \leq  O (d^{-1/2} ). \label{eq:uniformGaussianbound}
\end{align}
Since
\begin{align*}
\beta= \frac{1}{2} \log \frac{a}{b}= \frac{a-b}{2b} - \frac{(a-b)^2 }{4b^2} + O \left( \frac{(a-b)^3}{b^3} \right).
\end{align*}
it follows that  $d \beta \theta= \mu^2/4 + O(a^{-1/2})$ and $d \beta^2  = \mu^2/4 + O(a^{-1/2})$. Note that by definition, $w_1=\mu^2/4.$ For any $x \in \reals$, 
define $x'_{\pm}$ such that $\sqrt{d \beta^2} x'_{\pm} \pm d \beta \theta = x \sqrt{w_1} \pm w_1.$  Then
 \begin{align*}
 \left\{   \frac{ W_{\pm}^1 \mp w_t }{  \sqrt{w_t }}\leq x \right\} = \left\{  \frac{ W_{\pm}^1 \mp d \beta \theta }{  \sqrt{d  \beta^2  }}\leq x' \right\}.
\end{align*}
Hence,
\begin{align}
& \sup_x \bigg|     \prob{  \frac{W_{\pm}^{1}  \mp w_{1}  }{  \sqrt{ w_1} } \leq x } - \prob{Z \leq x}    \bigg|  \nonumber \\
& = \sup_x \bigg|    \prob{\frac{ W_{\pm}^1 \mp d \beta \theta }{  \sqrt{d  \beta^2  }}  \leq x' }  - \prob{Z \leq x} \bigg| \nonumber \\
&\le \sup_x \bigg|    \prob{\frac{ W_{\pm}^1 \mp d \beta \theta }{  \sqrt{d  \beta^2  }}  \leq x' }  - \prob{Z \leq x'} + \sup_x \bigg|  \prob{Z \leq x'} - \prob{Z \leq x} \bigg|  \nonumber\\
& \overset{(a)}{\le}  O (d^{-1/2} )  + \sup_x \bigg|  \prob{Z \leq x'} - \prob{Z \leq x} \bigg|  \nonumber \\
& \le O (d^{-1/2} ), \label{eq:triangleGaussian}
\end{align}
where $(a)$ holds due to \prettyref{eq:uniformGaussianbound}; the last inequality holds because $|x'-x| \le O \left( (x+1) d^{-1/2} \right)$ and hence,
\begin{align*}
\sup_x \bigg|  \prob{Z \leq x'} - \prob{Z \leq x} \bigg| \le  \sup_x \frac{1}{\sqrt{2\pi }} |x'-x| \max \left\{  \eexp^{-x^2/2}, \eexp^{-(x')^2/2 } \right\} =O (d^{-1/2} ).
\end{align*} 
Therefore, \prettyref{eq:gaussianlimitw} holds for $t=1$.

In view of \prettyref{eq:recursionPhi} and \prettyref{eq:recursionPsi}, $\Phi^t_u$ and $\Psi^t_u$ satisfy the same recursion. Moreover, by definition,
$v_t$ and $w_t$ also satisfy the same recursion. Thus, to finish the proof of the lemma, it suffices to show that:
 suppose \prettyref{eq:gaussianlimit}
holds for  a fixed $t$, then it also holds for $t+1.$ Also,
by symmetry, $Z^{t+1}_+$ has the same distribution as  $ -Z^{t+1}_-$, so it is enough to show \prettyref{eq:gaussianlimit}
holds for $Z^{t+1}_-$.

%
%

Notice that $Z^{t+1}_{ -}    = \sum_{i=1}^{N_d } Y_i, $
where $N_d \sim \Pois(d)$ is independent of $\{Y_i\}$; $\{Y_i\}$ are i.i.d.\ such that $ Y_i=F(Z_+^{t} +U )$ with probability $b/(a+b)$ and $Y_i=F(Z_-^{t} -U )$ with probability $a/(a+b)$.
Thus, $\expect{Z^{t+1}_{ -}  }= d \expect{Y_1}$ and $\var \left(Z^{t+1}_{ -}  \right) = d \expect{Y_1^2}$.
In view of  \prettyref{lmm:Poisson_BE}, we get that
\begin{align}
\sup_x \bigg|     \prob{  \frac{Z_-^{t+1}  - d \expect{Y_1}  }{  \sqrt{ d \expect{Y_1^2} } } \leq x } - \prob{Z \leq x}    \bigg|  = O \left(  \frac{ d \expect{ |Y_1|^3 } }{ (d \expect{Y_1^2} )^{3/2} } \right). \label{eq:gaussianconvergence}
\end{align}
It follows from \prettyref{lmm:meanvarianceZ} that
\begin{align*}
d \expect{Y_1} & = - \frac{ \mu^2 }{4} \expect{ \tanh(Z_+^t + U  ) }+ O( a^{-1/2} ). \\
 d \expect{Y_1^2} & =  \frac{ \mu^2 }{4} \expect{ \tanh(Z_+^t + U  ) }+ O( a^{-1/2} ).
\end{align*}
Using the area rule of expectation, we have that
\begin{align*}
& \expect{ \tanh(Z_+^t + U  ) } \\
& = \int_{0}^{1} \tanh'(t) \prob{ Z_+^t + U \ge t }  \diff t - \int_{-1}^{0} \tanh'(t)    \prob{ Z_+^t + U \le t } \\
& = \int_{0}^{1} \tanh'(t) \prob{ v_t + \sqrt{v_t} Z + U \ge t }  \diff t - \int_{-1}^{0} \tanh'(t)    \prob{ v_t + \sqrt{v_t} Z + U \le t }  + O(a^{-1/2}) \\
& = \expect{ \tanh ( v_t+ \sqrt{v_t} Z + U ) } +  O(a^{-1/2}). 
\end{align*}
where the second equality follows from  the induction hypothesis and the fact that $|\tanh'(t)| \le 1$. 
Hence,
$d \expect{Y_1} = -v_{t+1} + O( a^{-1/2} ) $ and $d \expect{Y_1^2}=v_{t+1} + O( a^{-1/2} )$. 
Moreover, since $F$ is monotone, it follows that $|F(x)| \le \max\{ |F(\infty)|, |F(-\infty)| \} =\beta $ and
thus $d \expect{ |Y_1|^3 } \le d \beta^3=  O(a^{-1/2})$.
As a consequence, in view of \prettyref{eq:gaussianconvergence} and following the similar argument as \prettyref{eq:triangleGaussian}, 
we get that \prettyref{eq:gaussianlimit} holds for $Z^{t+1}_-$.
 \end{proof}

 We are about to prove \prettyref{thm:accuracy} based on \prettyref{lmm:gaussiandensityevolution}.
 Before that, we need a lemma showing that $h$ is monotone.

 \begin{lemma}\label{lmm:hmonotone}
 $h(v)$ is continuous on $[0, \infty)$ and  $0 \le h'(v) \le 1$ for $v \in (0, +\infty)$.
 \end{lemma}
 \begin{proof}
 By definition,
 \begin{align*}
 h(v) = (1-\alpha) \expect{ \tanh \left( v + \sqrt{v} Z + \frac{1}{2} \log \frac{1-\alpha}{\alpha}  \right)   } +  \alpha \expect{ \tanh \left( v + \sqrt{v} Z - \frac{1}{2} \log \frac{1-\alpha}{\alpha}  \right)   }.
 \end{align*}
 Since $|\tanh(x)| \le 1$, the continuity of $h$ follows from the dominated convergence theorem. We next show $h'(v)$ exists for $v \in (0, \infty)$.
 Fix $c \in \reals$ and let $g(v)= \expect{\tanh ( v + \sqrt{v} Z + c  )}$.
 Notice that $\tanh'(x+ \sqrt{x} Z+c) = ( 1- \tanh^2 ( x + \sqrt{x} Z + c  )) ( 1+ x^{-1/2} Z/2 )$ for $x \in (0, \infty)$, and
 \begin{align*}
  \big| \left( 1- \tanh^2 ( x + \sqrt{x} Z + c  ) \right) ( 1+ x^{-1/2} Z/2 ) \big| \le  1+ x^{-1/2} |Z| /2.
 \end{align*}
Since $|Z|$ is integrable, by the dominated convergence theorem, $\expect{\tanh'(x+ \sqrt{x} Z+c)}$ exists and is continuous in $x$.
Therefore, $x \to \expect{\tanh'(x+ \sqrt{x} Z+c)}$ is integrable over $x \in (0,\infty)$. It follows that
\begin{align*}
g(v)= \expect{ \tanh(c) + \int_{0}^{v}  \tanh'(x + \sqrt{x} Z +c ) \diff x } =\tanh(c) + \int_0^v \expect{\tanh'(x + \sqrt{x} Z +c )} \diff x,
\end{align*}
where the second equality holds due to Fubini's theorem. Hence,
 \begin{align*}
 g'(v) = \expect{ \left( 1- \tanh^2 ( v + \sqrt{v} Z + c  ) \right) ( 1+ v^{-1/2} Z/2 ) }.
 \end{align*}
 Using the integration by parts, we can get that
 \begin{align*}
 & \expect{ \left( 1- \tanh^2 ( v + \sqrt{v} Z + c  ) \right)   \sqrt{v} Z  } \\
  & =\int_{-\infty}^{\infty}  (1-\tanh^2( v+ x +c ) \frac{1}{\sqrt{2\pi v} }  \eexp^{-x^2/2v} \diff x \\
  & = -v  \int_{-\infty}^{\infty}  (1-\tanh^2( v+ x +c )   \left( \frac{1}{\sqrt{2\pi v} } \eexp^{-x^2/2v} \right)'\diff x \\
  & = -v (1-\tanh^2( v+ x +c )  \frac{1}{\sqrt{2\pi v} }  \eexp^{-x^2/2v} \bigg|_{-\infty}^{+\infty} + v
  \int_{-\infty}^{\infty}  (1-\tanh^2( v+ x +c )'   \frac{1}{\sqrt{2\pi v} } \eexp^{-x^2/2v} \diff x \\
  &= - 2 v\expect{\tanh(v+\sqrt{v} Z+c) (1-\tanh^2( v+ \sqrt{v} Z +c  ) ) }.
 \end{align*}
The last two displayed equations yield that
\begin{align*}
g'(v) =  \expect{  \left( 1- \tanh ( \sqrt{v} Z + v +c ) \right) \left( 1- \tanh^2( \sqrt{v} Z + v +c)  \right) } .
\end{align*}
It follows that
\begin{align*}
h'(v) = \expect{  \left( 1- \tanh ( \sqrt{v} Z + v + U  ) \right) \left( 1- \tanh^2( \sqrt{v} Z + v +U )  \right) }.
\end{align*}
Thus $h'(v) \ge 0$. Finally, we show $h'(v) \le 1$. We need the following
equality: For $k \in \naturals$,
\begin{align}
\expect{\tanh^{2k} ( \sqrt{v} Z  + v + U ) } = \expect{\tanh^{2k-1} ( \sqrt{v} Z  + v +U ) }, \label{eq:symmetricequality}
\end{align}
which immediately implies that
\begin{align*}
h'(v) = \expect{  \left( 1- \tanh^2( \sqrt{v} Z + v +U )  \right) ^2 } \le 1.
\end{align*}
To prove \prettyref{eq:symmetricequality}, we need to introduce the notation of symmetric random variables \cite{Urbanke08,Montanari05}.
A  random variable $X$ is said to be symmetric if it takes values in $(-\infty, + \infty)$ and
\begin{align}
\expect{g(X)} = \expect{g(-X) \eexp^{-2X} }, \label{eq:defsymmetric}
\end{align}
for any real function $g$ such that at least one of the expectation values exists.
It is easy to check by definition that $\sqrt{v} Z + v$ and $U$ are symmetric. Moreover, one can check that
a sum of two independent, symmetric random variables is symmetric. Thus, $\sqrt{v} Z+ v + U$ is symmetric.
As shown in \cite[Lemma 3]{Montanari05}, if $X$ is symmetric, then $ \expect{\tanh^{2k} ( X ) } = \expect{\tanh^{2k-1} ( X ) }$. Specifically,
by plugging $g(x) = \tanh^{2k}(x)$ and $g(x)=\tanh^{2k-1}(x)$ into \prettyref{eq:defsymmetric}, we have that
\begin{align*}
\expect{ \tanh^{2k}(X)} &=\expect{ \tanh^{2k}( -X) \eexp^{-2X }} = \expect{ \tanh^{2k}( X) \eexp^{-2X }}, \\
\expect{ \tanh^{2k-1}(X)} &=\expect{ \tanh^{2k-1}( -X) \eexp^{-2X }} = -\expect{ \tanh^{2k-1}( X) \eexp^{-2X }} .
\end{align*}
It follows that
\begin{align*}
\expect{ \tanh^{2k}(X)}  = \frac{1}{2} \expect{ \tanh^{2k}(X) ( 1+ \eexp^{-2X} ) } = \frac{1}{2}     \expect{ \tanh^{2k-1}(X) ( 1- \eexp^{-2X} ) } =  \expect{ \tanh^{2k-1} (X) }.
\end{align*}
 \end{proof}


 \begin{proof}[Proof of \prettyref{thm:accuracy}]
 In view of \prettyref{lmm:gaussiandensityevolution},
 \begin{align*}
  \lim_{n \to \infty} \prob{ \Gamma_u^t \ge0 | \tau_u=-} =  \lim_{n \to \infty} \prob{ \Gamma_u^t  \le 0 | \tau_u =+ }  = (1-\alpha) Q \left(  \frac{v_t + \gamma } { \sqrt{v_t} }\right) + \alpha Q \left(  \frac{v_t - \gamma } { \sqrt{v_t} }\right).
 \end{align*}
 Hence, it follows from \prettyref{lmm:optBPcondition} that
 \begin{align*}
 \lim_{n \to \infty} p_{G_n}( \hat{\sigma}_{\rm BP}^t ) = \lim_{n \to \infty} q_{T^t}^\ast =  1- \expect{ Q \left(  \frac{v_t + U }{ \sqrt{ v_t } }\right) }.
 \end{align*}
 
 We prove that $v_{t+1} \ge v_{t}$ for $t \ge 0$ by induction. Recall that $v_0=0 \le v_1= (1-2\alpha)^2 \mu^2/4 =\mu^2 h(v_0)/4 $.
 Suppose $v_{t+1} \ge v_{t}$ holds; we shall show the claim also holds for $t+1$. In particular, since $h$ is continuous on $[0, \infty)$ and
 differential on $(0, \infty)$, it follows from the mean value theorem that
 \begin{align*}
 v_{t+2} - v_{t+1} = \frac{\mu^2}{4} \left( h (v_{t+1} ) - h(v_{t} ) \right) =  \frac{\mu^2}{4} h' (x),
 \end{align*}
 for some $x \in (v_{t}, v_{t+1} ) $.  \prettyref{lmm:hmonotone} implies that $h'(x) \ge 0$ for $x \in (0, \infty)$,
 it follows that $v_{t+2} \ge v_{t+1}$.  Hence, $v_t$ is non-decreasing in $t$. Next we argue that $v_t \le \underline{v}$
 for all $ t \ge 0$ by induction, where $\underline{v}$ is the smallest fixed point of $v=\frac{\mu^2}{4} h(v)$. For the base case, $v_0=0 \le \underline{v}$.
 If $v_t \le \underline{v}$, then by the monotonicity of $h$, $v_{t+1}=\frac{\mu^2}{4}  h (v_{t} )  \le \frac{\mu^2}{4}  h ( \underline{v} ) =\underline{v}.$ 
 Thus, $\lim_{t \to \infty} v_t$ exists and $\lim_{t \to \infty} v_t= \underline{v}$. 
Therefore,
\begin{align*}
\lim_{t\to \infty}  \lim_{n \to \infty} p_{G_n}( \hat{\sigma}_{\rm BP}^t )  = \lim_{t \to \infty} \lim_{n \to \infty} q_{T^t}^\ast
= 1- \expect{ Q \left(  \frac{\underline{v} + U }{ \sqrt{ \underline{v}} }\right) }.
\end{align*}
Next, we prove the claim for $p^\ast_{G_n}$. In view of \prettyref{lmm:gaussiandensityevolution},
 \begin{align*}
  \lim_{n \to \infty} \prob{ \Lambda_u^t \ge0 | \tau_u=-} =  \lim_{n \to \infty} \prob{ \Lambda_u^t  \le 0 | \tau_u =+ }  = (1-\alpha) Q \left(  \frac{w_t + \gamma } { \sqrt{w_t} }\right) + \alpha Q \left(  \frac{w_t - \gamma } { \sqrt{w_t} }\right).
 \end{align*}
  Hence, it follows from \prettyref{lmm:accuracyupperbound} that
 \begin{align*}
 \limsup_{n \to \infty} p_{G_n}^\ast \le  \lim_{n\to \infty} p_{T^t}^\ast = 1-  \expect{ Q \left(  \frac{w_t+ U }{ \sqrt{ w_t } }\right) } .
 \end{align*}
Recall that $w_1=\mu^2/4 \ge w_t$.  By the same argument of proving $v_t$ is non-decreasing, one can show that $w_t$ is non-increasing in $t$.
Also, by the same argument of proving $v_t$ is upper bounded by $\underline{v}$, one can show that $w_t$ is lower bounded by $\overline{v}$,
where where $\overline{v}$ is the largest fixed point of $v=\frac{\mu^2}{4} h(v)$. Thus, $\lim_{t \to \infty} w_t$ exists and $\lim_{t \to \infty} w_t= \overline{v}$. 
Therefore,
\begin{align*}
\lim_{t\to \infty}  \limsup_{n \to \infty} p_{G_n}^\ast  \le \lim_{t\to \infty}  \lim_{n\to \infty} p_{T^t}^\ast
 = 1- \expect{ Q \left(  \frac{\overline{v} + U }{ \sqrt{ \overline{v}} }\right) }.
\end{align*}

Finally, notice that
\begin{align*}
w_{t+1}- v_{t+1} = \frac{\mu^2}{4} \left( h(w_t) - h(v_t) \right) \le   \frac{\mu^2}{4} (w_t-v_t),
\end{align*}
where the last inequality holds because $0\le h'(x) \le 1$. If $|\mu| <2$, then $\mu^2/4 \le 1-\epsilon$ for some $\epsilon>0$.
Hence, $(w_{t+1}- v_{t+1})\le(1-\epsilon) (w_t-v_t)$. Since $w_1-v_1=\mu^2 \alpha(1-\alpha)$, it follows that $\lim_{t\to \infty} (w_t-v_t)=0$ and
thus $\underline{v}=\overline{v}$. If instead $|\mu| \ge C$ for some sufficiently large constant $C$ or $\alpha \le \alpha^\ast$ for some sufficiently small constant $0<\alpha^\ast<1/2$,
then  it follows from \prettyref{thm:main} and \prettyref{thm:regularaccurateSI} that $\lim_{t\to \infty}  \lim_{n\to \infty} p_{T^t}^\ast  =  \lim_{t\to \infty}  \lim_{n\to \infty} q_{T^t}^\ast$.
As a consequence,  $\underline{v}=\overline{v}.$

 \end{proof}




\bibliographystyle{abbrv}
\bibliography{graphical_combined,other_refs,other_refs2}

\end{document}